\documentclass[journal]{IEEEtran}
\usepackage{amsmath,amsfonts}
\usepackage[linesnumbered,ruled,vlined]{algorithm2e}
\usepackage{array}
\usepackage[caption=false,font=normalsize,labelfont=sf,textfont=sf]{subfig}
\usepackage{textcomp}
\usepackage{stfloats}
\usepackage{url}
\usepackage{verbatim}
\usepackage{graphicx}
\usepackage{cite}
\usepackage{hyperref}
\usepackage{setspace}
\usepackage{amsthm}
\usepackage{bm}
\usepackage{bbm}
\usepackage{booktabs}
\usepackage{xcolor}
\usepackage{soul}
\definecolor{Tan1}{RGB}{255,165,79}
\sethlcolor{Tan1}

\newtheorem{definition}{Definition}
\newtheorem{proposition}{Proposition}

\begin{document}

\title{Deep into The Domain Shift: Transfer Learning through Dependence Regularization}

\author{Shumin Ma, Zhiri Yuan, Qi Wu, Yiyan Huang, Xixu Hu, Cheuk Hang Leung, Dongdong Wang, Zhixiang Huang
\thanks{Shumin Ma is with Guangdong Provincial Key Laboratory of Interdisciplinary Research and Application for Data Science, BNU-HKBU United International College; Division of Science and Technology, BNU-HKBU United International College, Zhuhai 519087, China.}
\thanks{Zhiri Yuan, Yiyan Huang, Xixu Hu, and Cheuk Huang Leung are with the CityU-JD Digits Joint Laboratory in Financial Technology and Engineering and the School of Data Science at the City University of Hong Kong, Hong Kong.}
\thanks{Qi Wu is with the School of Data Science, The CityU-JD Digits Joint Laboratory in Financial Technology and Engineering, and the Institute of Data Science at the City University of Hong Kong. He is also with The Laboratory for AI-Powered Financial Technologies Limited, Hong Kong.}
\thanks{Dongdong Wang and Zhixiang Huang are with JD Digits Technology, Beijing, China.}
\thanks{Corresponding author: Qi Wu.}}



\maketitle

\begin{abstract}
Classical Domain Adaptation methods acquire transferability by regularizing the overall distributional discrepancies between features in the source domain (labeled) and features in the target domain (unlabeled). They often do not differentiate whether the domain differences come from the marginals or the dependence structures. In many business and financial applications, the labeling function usually has different sensitivities to the changes in the marginals versus changes in the dependence structures. Measuring the overall distributional differences will not be discriminative enough in acquiring transferability. Without the needed structural resolution, the learned transfer is less optimal. This paper proposes a new domain adaptation approach in which one can measure the differences in the internal dependence structure separately from those in the marginals. By optimizing the relative weights among them, the new regularization strategy greatly relaxes the rigidness of the existing approaches. It allows a learning machine to pay special attention to places where the differences matter the most. Experiments on three real-world datasets show that the improvements are quite notable and robust compared to various benchmark domain adaptation models.
\end{abstract}


\section{Introduction}
\IEEEPARstart{U}{nsupervised} domain adaptation emerges when one estimates a prediction function in a given target domain without any labeled samples by exploiting the knowledge available from a source domain where labels are known. The critical step in the transfer is to extract feature representations that are invariant across domains. A large body of work learns the domain-invariant feature representations by minimizing various metrics on the feature distributions between domains:  Proxy-$\mathcal{A}$ distance \cite{ganin2016domain}, total variation distance \cite{tzeng2017adversarial, zhao2018adversarial, liu2019transferable}, maximum mean discrepancy (MMD, \cite{gretton2012kernel, long2015learning, long2017deep}), Wasserstein distance \cite{courty2016optimal, shen2018wasserstein}, etc. It is worth noting that in most of the literature, the domain invariance is measured on the overall feature distributions between the source and the target domains.

\begin{figure*}[!htbp]
\centering
 \includegraphics[width=0.6\columnwidth]{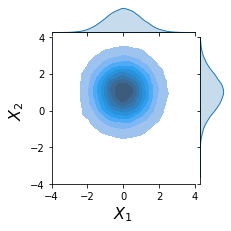}
 \includegraphics[width=0.6\columnwidth]{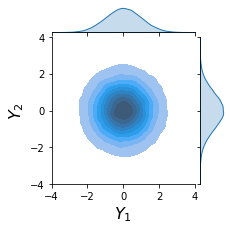}
 \includegraphics[width=0.6\columnwidth]{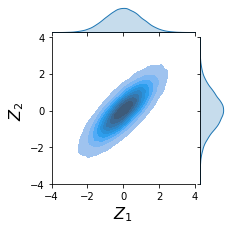}
 \caption{Visualization of three 2-d Gaussian distributions: ($P^{\mathbf{X}}$) 
 $N(
\begin{bmatrix}
0\\
1
\end{bmatrix},
\begin{bmatrix}
1 & 0\\
0 & 1
\end{bmatrix}
)$, ($P^{\mathbf{Y}}$)
$N(
\begin{bmatrix}
0\\
0
\end{bmatrix},
\begin{bmatrix}
1 & 0\\
0 & 1
\end{bmatrix}
)$, ($P^{\mathbf{Z}}$)
 $N(
\begin{bmatrix}
0\\
0
\end{bmatrix},
\begin{bmatrix}
1 & \sqrt{1-e^{-1}}\\
\sqrt{1-e^{-1}} & 1
\end{bmatrix}
)$. $P^{\mathbf{X}}$ and $P^{\mathbf{Y}}$ differ only in the 2nd marginal distribution, while $P^{\mathbf{Z}}$ and $P^{\mathbf{Y}}$ differ in the dependence structure. One can check that the KL divergence between $P^{\mathbf{X}}$ and $P^{\mathbf{Y}}$ is $1/2$ and it is the same as the KL divergence between $P^{\mathbf{Z}}$ and $P^{\mathbf{Y}}$. That is, one single divergence value ($1/2$ in this example) can not distinguish whether the distribution difference comes from the marginals or the dependence structure.}
 \label{Fig_KL example}
\end{figure*}

However, the overall feature distribution difference encodes both the marginals' distinctions and the dependence difference into one metric value, making it hard to identify whether the domain difference comes from the marginals or the dependence structure. As a motivation example, we use Figure \ref{Fig_KL example} to clarify this point. In Figure \ref{Fig_KL example}, there are three random vectors ($\mathbf{X}$, $\mathbf{Y}$ and $\mathbf{Z}$) that follow three different Gaussian distributions ($P^\mathbf{X}$, $P^\mathbf{Y}$ and $P^\mathbf{Z}$, with the details shown in the caption) respectively. It can be observed that $P^{\mathbf{X}}$ and $P^{\mathbf{Y}}$ only differ in the 2nd marginal distribution where $P^{\mathbf{X}}_2$ is $N(1,1)$ and $P^{\mathbf{Y}}_2$ is $N(0,1)$. Namely, the overall distribution difference is completely caused by the marginal distinction. However, for $P^{\mathbf{Y}}$ and $P^{\mathbf{Z}}$, it is well observed that the marginals are the same while the covariance matrix differs. Namely, the overall distribution difference over the latter two distributions is controlled by the dependence difference. In general cases, any two distributions can differ in the marginal distributions and the dependence structure simultaneously. One can check that the KL divergence between $P^{\mathbf{X}}$ and $P^{\mathbf{Y}}$ is the same with that of $P^{\mathbf{Z}}$ and $P^{\mathbf{Y}}$, which is $1/2$ in the example. That is to say, a single divergence value cannot distinguish between marginal difference and dependence difference. Furthermore, the overall feature distributions' divergence sums the marginals' and dependence differences up together in a relatively fixed manner. Such rigidness is undesirable when the prediction function has different sensitivities to the changes in the marginals versus changes in the dependence structure. Especially in many financial and business applications, the prediction function may depend heavily on the dependence structure of the features. For example, during the market crash, the stocks show remarkably synchronous co-movement, implying a stronger dependence than regular times. Such observations motivate us to relax the binding between the marginals' and dependence differences in one metric value and pay attention to the difference term that matters the most in the transfer.

We propose a new domain adaptation model to regularize the domain differences in the dependence structure via the copula distance, separately from the marginal divergence. The idea is inspired by Sklar's Theorem \cite{sklar1959fonctions} which states that any multivariate distribution can be decomposed as the product of marginal distributions and a copula function, and vice versa. It explicitly shows that the copula function, together with the marginal distributions, is sufficient to recover the original multivariate distribution. The efficacy and versatility of our approach are demonstrated with real-world classification and regression problems.

The contributions of this paper are summarized as follows. (1) We propose a novel deep domain adaptation framework that allows more flexibility to combine the marginal difference and the dependence difference into a regularizer. (2) We explore the structural properties of the copula distance that guarantee the algorithm convergence of our approach. (3) Our proposed model proves its efficiency on two novel datasets (a large-scale retail credit dataset and an intra-day equity price dataset) and one standard UCI dataset.

\section{Related work}

Due to the ability of deep neural nets to learn rich feature representations, deep domain adaptation models have focused on using these networks to learn invariant representations, i.e., intermediate features whose distributions are the same in the source and the target domains, while at the same time achieving a small prediction error on the source domain. The hope is that the learned representation, together with the hypothesis learned from the source domain, can generalize to the target domain. There are many ways to measure the domain invariance \cite{zhao2020review}. \cite{duan2012domain} uses $L_2$ norm to directly align features of different domains. \cite{long2015learning} proposes a Deep Adaptation Network (DAN) architecture that embeds the feature representations in a reproducing kernel Hilbert space (RKHS) and reduces the domain discrepancy through multi-kernel MMD. \cite{ganin2016domain} proposes a domain adversarial neural network (DANN) to learn the domain-invariant features with a min-max formulation. \cite{zhao2019learning} characterizes a fundamental tradeoff between learning invariant representations and achieving a small joint error on both domains when the marginal label distributions differ from the source to the target. Furthermore, \cite{chen2020subspace} tackles the knowledge transfer problem under the generalized covariate shift condition by Bregman divergence. \cite{wang2019domain} proposes a novel neural embedding matching method by enforcing consistent class-wise cross-domain instance distributions in the embedding space. \cite{zhang2019guide} proposes a two-stage progressive training strategy to learn invariant, discriminative, and domain-agnostic subspace. Another line of work proposes to reduce the domain discrepancy through minimizing the optimal transport loss between the source and target distributions \cite{courty2016optimal, shen2018wasserstein}. For example, \cite{shen2018wasserstein} minimizes the empirical Wasserstein distance between the source and target samples. However, these methods focus on the overall distribution discrepancy and often do not differentiate whether the domain differences come from the marginals or the dependence structure.

To explicitly encode the dependence difference in the domain adaptation framework, \cite{sun2016return} proposes a correlation alignment (CORAL) model to measure the dis-similarity by the Frobenius norm of the covariance matrices from the two domains. \cite{sun2016deep} further combines CORAL with deep neural networks and verifies its effectiveness through extensive experiments on the standard benchmark datasets. \cite{zhang2018aligning} matches distributions by aligning the RKHS covariance matrices across domains. \cite{chen2020homm} integrates the MMD and CORAL into a unified framework and exploits the higher-order statistics for domain alignment. Our approach is more general in that it separates the marginals' divergence and the dependence difference and integrates them into one regularizer. It facilitates us to detect the changes in the marginals and the dependence structure simultaneously.

Our work is closely related to copulas. Copulas have been successfully used in many deep learning methods. \cite{zhao2020missing} performs the missing value imputation by developing a semi-parametric algorithm to estimate copula parameters from incomplete mixed data. \cite{marti2017exploring} and \cite{tran2017unsupervised} propose to summarize and measure the pairwise correlations between variables, which is shown to well capture the various dependence patterns. By assuming the underlying features have a specific structure, \cite{lopez2012semi} adopts non-parametric vine copula for semi-supervised domain adaptation problems. \cite{letizia2020segmented} uses copula to generate dependent data in a segmented way. \cite{xuan2017doubly} incorporates copula into a doubly nonparametric sparse nonnegative matrix factorization framework. \cite{quan2019survey} well documents the literature that takes advantage of copulas to model the correlation in multivariate data in smart grid. We extend the idea of copulas and construct a copula-based divergence measure to quantify the dependence difference. Our proposed measure incorporates more statistical information than the commonly-used covariance matrices that only capture the second-order statistics.

\section{Copula distance}\label{SecCD}

Suppose that a $d$-dimensional random variable $\mathbf{X} = [X_1, \dots, X_d]$ is characterized by the cumulative distribution function (CDF) $P$ and its density function is denoted by $p$. Sklar's theorem \cite{sklar1959fonctions} states that there exists a copula $C$ such that $P(x_1, \dots, x_d) = C(P_1(x_1),\dots,P_d(x_d))$, where $P_i(\cdot)$ is the marginal CDF. Furthermore, any continuous density function $p(x_1, \dots, x_d)$ can be written in terms of univariate marginal density functions $\{p_i(x_i)\}_{i=1}^d$ and a unique copula density function $c:[0,1]^d\rightarrow \mathbb{R}$ which characterizes the dependence structure: 
\begin{equation}\label{EqCopulaDensity}
p(x_1, \dots, x_d) = c(u_1,\ldots,u_d)\times \prod\nolimits_{i=1}\nolimits^d p_i(x_i),
\end{equation} 
where $ u_i:=P_i(x_i),\, \forall\,1\leq i\leq d$.

One can use the copula function to extract a clean quantification of the dependence strength between any components [$X_i$, $X_j$] in the random vector $\mathbf{X}$. For example, the \textit{mutual information} between $X_i$ and $X_j$, a well-known dependence measure in information theory, is equivalent to the negative entropy of the copula function, namely,
\begin{equation*}
\small
\mathcal{H}_{KL}(P_{ij},P_i P_j) = \int_0^1 \int_0^1 c_{ij}(u_i,u_j)\log c_{ij}(u_i,u_j) \mathrm{d}u_i \mathrm{d}u_j.
\end{equation*}
Here, $\mathcal{H}_{KL}$ denotes the Kullback-Leibler (KL) divergence, $P_{ij}$ is the joint distribution of $[X_i,X_j]$, and $P_i$, $P_j$ represent the marginal distributions. $c_{ij}(u_i,u_j)$ is short for the density function $c(1,\ldots,1,u_i,1,\ldots,1,u_j,1,\ldots,1)$ where the $i$-th and the $j$-th arguments are $u_i$ and $u_j$ respectively, and the other arguments are all $1$. It should be noted that mutual information is a special case of a more general dependence measurement framework \cite{marti2017exploring} that computes the distance $\mathcal{H}(P_{ij},P_i P_j)$ with any divergence measure $\mathcal{H}$. We list a few examples below where $\mathcal{H}$ takes $\chi^2$ distance \cite{NIPS2016_cedebb6e}, Hellinger distance \cite{NIPS2016_cedebb6e}, and $\alpha$-divergence \cite{pantazis2022cumulant} (the detailed derivation can be found in the Appendix):
\begin{itemize}
\setlength{\itemsep}{0pt}
\setlength{\parsep}{0pt}
\setlength{\parskip}{0pt}
\item[-] $\mathcal{H}_{\chi^2}(P_{ij},P_i P_j) = \int_0^1 \int_0^1 (c_{ij}^2(u_i,u_j)-1)\mathrm{d}u_i \mathrm{d}u_j$.
\item[-] $\mathcal{H}_{H}(P_{ij},P_i P_j) = \int_0^1 \int_0^1 [\sqrt{c_{ij}(u_i,u_j)}-1]^2\mathrm{d}u_i \mathrm{d}u_j$.
\item[-] $\mathcal{H}_{\alpha}(P_{ij},P_i P_j) = \frac{1}{1-\alpha^2}\int_0^1 \int_0^1 [1-{c_{ij}(u_i,u_j)}^{-\frac{\alpha+1}{2}}]c_{ij}(u_i,$ $u_j)\mathrm{d}u_i \mathrm{d}u_j$.
\end{itemize}

For most of the divergence measures $\mathcal{H}$, it can be proved that $\mathcal{H}(P_{ij},P_i P_j)$ is a function of the copula densities (see Appendix). For a given random vector $\mathbf{X}$, if $\mathcal{H}(P_{ij},P_i P_j)$ provides a clean measure of any of its component pairs [$X_i$, $X_j$], one can run through the indexes, give them different weights $\beta_{ij}$, and sum over all the weighted $\mathcal{H}$s. The resulting quantity would allow one to look into the complete internal structure of $\mathbf{X}$ along the direction of any user-defined attention angle using any measure $\mathcal{H}$. This observation motivates us to use such a measure to quantify the difference in internal dependencies between two random vectors $\mathbf{X}$ and $\mathbf{Y}$. Below is its precise definition.

\begin{table*}[!hbp]
\centering
\small
\caption{For any two adjacent distributions in Figure 1, the overall distribution divergence, the $1^{st}$-d and the $2^{nd}$-d marginal distinction, together with the copula distance (CD) are recorded. Specifically, we take the KL divergence, Jenson-Shannon (JS) divergence and MMD as the distribution divergence measure for illustration purposes. }
\begin{tabular}{lcccccccc} 
\toprule
& \multicolumn{4}{c}{$P^{\mathbf{X}}$ vs. $P^{\mathbf{Y}}$} & \multicolumn{4}{c}{$P^{\mathbf{Z}}$ vs. $P^{\mathbf{Y}}$} \\
\cmidrule(lr){2-5} \cmidrule(lr){6-9}
      $\mathcal{H}$  & $\mathcal{H}(P^{\mathbf{X}},P^{\mathbf{Y}})$  & $\mathcal{H}(P^{\mathbf{X}}_1,P^{\mathbf{Y}}_1)$  & $\mathcal{H}(P^{\mathbf{X}}_2,P^{\mathbf{Y}}_2)$  & CD & $\mathcal{H}(P^{\mathbf{Z}},P^{\mathbf{Y}})$ & $\mathcal{H}(P^{\mathbf{Z}}_1,P^{\mathbf{Y}}_1)$  & $\mathcal{H}(P^{\mathbf{Z}}_2,P^{\mathbf{Y}}_2)$  & CD \\
\midrule
KL & 0.5 & 0 & 0.5 & 0 & 0.5 & 0 &  0 & 0.5 \\
JS & 0.5 & 0 & 0.5 & 0 & 0.859 & 0 & 0 & 0.859 \\
MMD & 0.107 & 0 & 0.175 & 0 & 0.042 & 0 &  0 & 0.042 \\
\bottomrule
\end{tabular}
\label{Tb_dist_example}
\end{table*}

\begin{definition}\label{Def_CD} (\textbf{Copula distance})
Let $\mathbf{X} = [X_1,\ldots,X_d]\in \mathbb{R}^d$ and $\mathbf{Y} = [Y_1,\ldots,Y_d] \in \mathbb{R}^d$ be two random vectors. Let $P_{ij}^{\mathbf{X}}$ and $P_{ij}^{\mathbf{Y}}$ be the cumulative joint distributions of any of their component pairs [$X_i$, $X_j$] and [$Y_i$, $Y_j$], respectively. Let $\mathcal{H}(\cdot,\cdot)$ be a measure of distribution difference. We define the copula distance between $[X_i,X_j]$ and $[Y_i,Y_j]$ ($\forall 1\leq i < j \leq d$) as:
\begin{equation*}
\resizebox{.95\linewidth}{!}{$
    \displaystyle
CD_{\mathcal{H}}([X_i,X_j],[Y_i,Y_j])=|\mathcal{H}(P^{\mathbf{X}}_{ij},P^{\mathbf{X}}_i P^{\mathbf{X}}_j) - \mathcal{H}(P^{\mathbf{Y}}_{ij},P^{\mathbf{Y}}_i P^{\mathbf{Y}}_j)|.
$}
\end{equation*}
The copula distance between random vectors $\mathbf{X}$ and $\mathbf{Y}$ w.r.t. the positive weights $\bm{\beta}=\{\beta_{ij} \}_{1 \leq i < j \leq d}$ is defined as:
\begin{equation}\label{EqCD}
\resizebox{.91\linewidth}{!}{$
    \displaystyle
CD_{\mathcal{H}}(\mathbf{X},\mathbf{Y};\bm{\beta}) = \sum \limits_{1\leq i < j \leq d} \beta_{ij} CD_{\mathcal{H}}([X_i,X_j],[Y_i,Y_j]).
$}
\end{equation}
\end{definition}

As explained before, although the copula distance seems irrelevant to the copulas at first glance, it is actually a function of the copula densities. In the following context, we will use a set of $d$ terms $\{\mathcal{H}(P^{\mathbf{X}}_i, P^{\mathbf{Y}}_i)\}_{i=1}^d$ to measure the marginals' distinctions and the copula distance $CD_{\mathcal{H}}(\mathbf{X},\mathbf{Y};\bm{\beta})$ to measure the dependence difference of any two random vectors $\mathbf{X}\in \mathbb{R}^d$ and $\mathbf{Y}\in \mathbb{R}^d$. To explain the concepts in more details, we use the distributions in Figure 1. We record the marginals' distinctions and dependence difference between each two of the three distributions under KL divergence (KL), Jenson-Shannon divergence (JS) and MMD in Table \ref{Tb_dist_example}.

When the copula density takes some specific function form, it can be proved that the copula distance has two nice properties: bounded and monotonic. These two properties guarantee the convergence of our proposed algorithms in the later sections. There are multiple parametric copula density functions (Gaussian copula, $t$ copula, Gumbel copula, etc.) that can be effectively incorporated into our model framework. In this work, we take the copula density in the form of Gaussian copula. Gaussian copula has been widely used in the financial literature because it is convenient to capture the dependence embedded in the random vector \cite{donnelly_embrechts_2010, FANG2013292, wang2014semiparametric}. A random vector $\mathbf{X} \in \mathbb{R}^d$ is said to have Gaussian copula with parameter $\Sigma \in \mathbb{R}^{d\times d}$ if the copula density function $c(u_1,\ldots,u_d) = |\Sigma|^{-\frac{1}{2}} \exp (-\frac{1}{2}\mathbf{x}^T(\Sigma^{-1}-\mathbf{I})\mathbf{x})$, where $\mathbf{x}:=[\Phi^{-1}(u_1),\ldots, \Phi^{-1}(u_d)]^T$ with $\Phi$ being the cumulative density function for the standard normal distribution. We state the boundedness and monotonicity under the Gaussian copula here and defer the proof to the Appendix.

\begin{proposition} \label{boundedness}
\textbf{(Boundedness)} The copula distance defined in Eq. \eqref{EqCD} is bounded when the divergence metric $\mathcal{H}$ is taken to be MMD distance, Wasserstein-2 distance or some of the $\phi$-divergence (including Jensen-Shannon distance, Hellinger distance and total variation distance).
\end{proposition}

\begin{proposition} \label{monotonicity}
\textbf{(Monotonicity)} Let $\Sigma^{\mathbf{X}}$ and $\Sigma^{\mathbf{Y}}$ be the Gaussian copula parameters for the random vectors $\mathbf{X}\in \mathbb{R}^d$ and $\mathbf{Y}\in \mathbb{R}^d$, respectively. Given all of the other entries in $\Sigma^{\mathbf{X}}$ and $\Sigma^{\mathbf{Y}}$ fixed, the copula distance $CD_{\mathcal{H}}([X_i,X_j],[Y_i,Y_j])$ is monotonically increasing with $|(\Sigma_{ij}^{\mathbf{X}})^2 - (\Sigma_{ij}^{\mathbf{Y}})^2|$, $\forall~1\leq i < j \leq d$. The monotonicity is satisfied by general probability distribution divergence measures, including MMD distance, Wasserstein-2 distance, and most of the commonly-used $\phi$-divergence (including KL divergence, $\chi^2$ distance, Hellinger distance, etc.).
\end{proposition}

\section{Unsupervised domain adaptation} \label{SecUDA}

\textbf{Notations.} In unsupervised domain adaptation, we are given a \textit{source} domain $\mathcal{D}_s = \{(\mathbf{x}_n^s,y_n^s)\}_{n=1}^{N_s}$ with $N_s$ labeled examples, and a \textit{target} domain $\mathcal{D}_t = \{\mathbf{x}_n^t\}_{n=1}^{N_t}$ with $N_t$ unlabeled examples. It is assumed that the two domains are characterized by different probability distributions, while they share the same feature space. In a classification task, the goal is to learn a transferable classifier to minimize the classification error on the target domain using all the given data.

Deep domain adaptation methods begin with a feature extractor that can be implemented by a neural network. The feature extractor is supposed to learn the domain-invariant feature representations from both domains. Specifically, the feature extractor learns a function $F(\mathbf{x}; \theta_f):\mathbb{R}^d \rightarrow \mathbb{R}^m$ that maps an instance to an $m$-dimensional representation with the network parameters $\theta_f$ (as illustrated in Figure \ref{Fig_network}). For simplicity, we denote the feature representation of a source instance $\mathbf{x}_n^s$ as $\mathbf{F}_n^s := F(\mathbf{x}_n^s; \theta_f)$, and that of a target instance $\mathbf{x}_n^t$ as $\mathbf{F}_n^t := F(\mathbf{x}_n^t; \theta_f)$. We define the source feature set $\mathcal{F}^s := \{\mathbf{F}_n^s\}_{n=1}^{N_s}$ and the target feature set $\mathcal{F}^t := \{\mathbf{F}_n^t\}_{n=1}^{N_t}$.

\begin{figure*}
\begin{center}
\includegraphics[height=.35\textwidth]{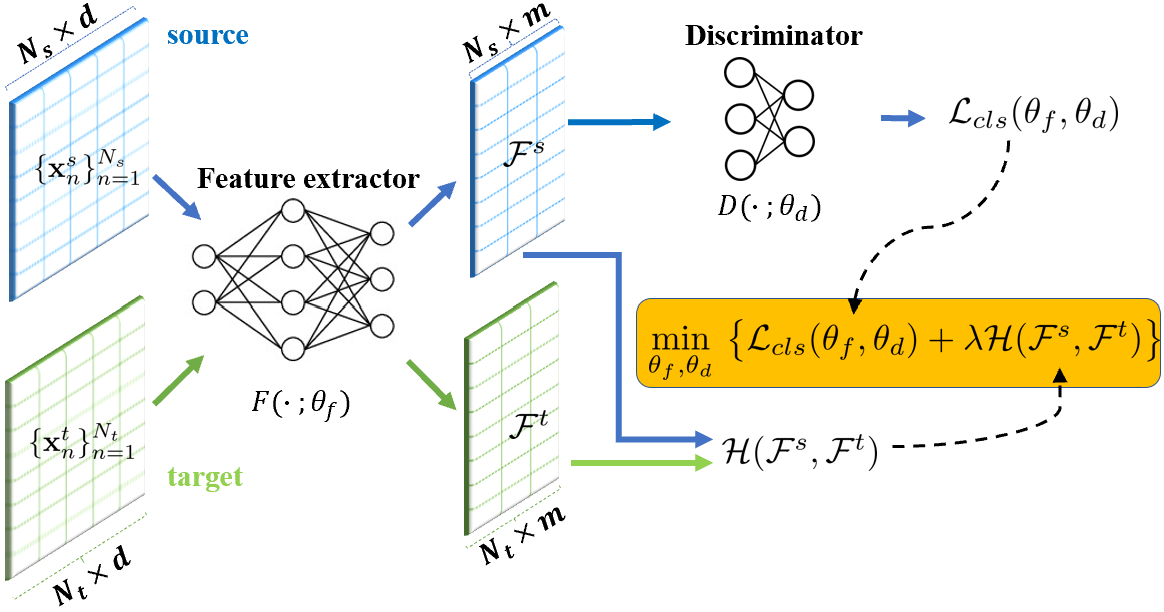}
\caption{The network flow in deep domain adaptation models. All samples, no matter from the source domain $\{\mathbf{x}_n^s\}_{n=1}^{N_s}$ (blue) or the target domain $\{\mathbf{x}_n^s\}_{n=1}^{N_t}$ (green), are fed into a feature extractor $F(\cdot;\theta_f)$ to extract features that are both discriminative and domain-invariant. The discriminative features are achieved by training a discriminator $D(\cdot;\theta_d)$ to minimize the loss function $\mathcal{L}_{cls}(\theta_f,\theta_d)$, while the domain invariance is measured by $\mathcal{H}(\mathcal{F}^s,\mathcal{F}^t)$.}
\label{Fig_network}
\end{center}
\end{figure*}

When domain adaptation is applied to a classification problem, the feature extractor is followed by a discriminator trained with samples from the source domain $\mathcal{D}_s$. Given the feature representations $\mathcal{F}^s$ computed by the feature extractor on the source domain, together with the labels $\{y_n^s\}_{n=1}^{N_s}$ ($y_n^s \in \{1,2,\ldots,l\}$), we can train a discriminator $D(\cdot ~;\theta_d):\mathbb{R}^m \rightarrow \mathbb{R}^l$ that is characterized by the parameters $\theta_d$. The discriminator loss function is defined as the cross-entropy between the predicted probabilistic distribution and the one-hot encoding of the class labels:
\begin{equation} \label{eq_src_error}
\small
\mathcal{L}_{cls}(\theta_f, \theta_d)  := -\frac{1}{N_s}\sum\limits_{n=1}\limits^{N_s} \sum\limits_{i=1}\limits^l \mathbf{1}(y_n^s = i)\cdot \log D\big(F(\mathbf{x}_n^s; \theta_f);\theta_d\big)_i,
\end{equation}
where $\mathbf{1}(\cdot)$ is the indicator function and $D(\cdot~;\theta_d)_i$ represents the $i$-th element in the predicted distribution $D(\cdot~;\theta_d)$.

Besides the discriminator that is trained to learn the discriminative features, the feature extractor is also followed by a discrepancy term that measures the difference between the source features $\mathcal{F}^s$ and the target features $\mathcal{F}^t$ to learn domain-invariant feature representations. To be specific, for a given discrepancy mesure $\mathcal{H}$, the empirical discrepancy between the source feature distribution and the target feature distribution is given by $\mathcal{H}(\mathcal{F}^s, \mathcal{F}^t)$. In literature, there have been multiple choices of the discrepancy measures $\mathcal{H}$, such as MMD distance \cite{long2015learning}, Wasserstein distance \cite{shen2018wasserstein}, JS distance \cite{goodfellow2014generative}, Proxy-$\mathcal{A}$ distance \cite{ben2007analysis}, etc. We call the latter few measures as \textit{adversarial distance} because they are implemented with a domain classifier to quantify the invariance between $\mathcal{F}^s$ and $\mathcal{F}^t$. We illustrate the commonly-chosen discrepancy measures in domain adaptation models in Figure \ref{Fig_distance}(a)-(b). 
\begin{figure*}[!htbp]
\begin{center}
\includegraphics[height=.4\textwidth]{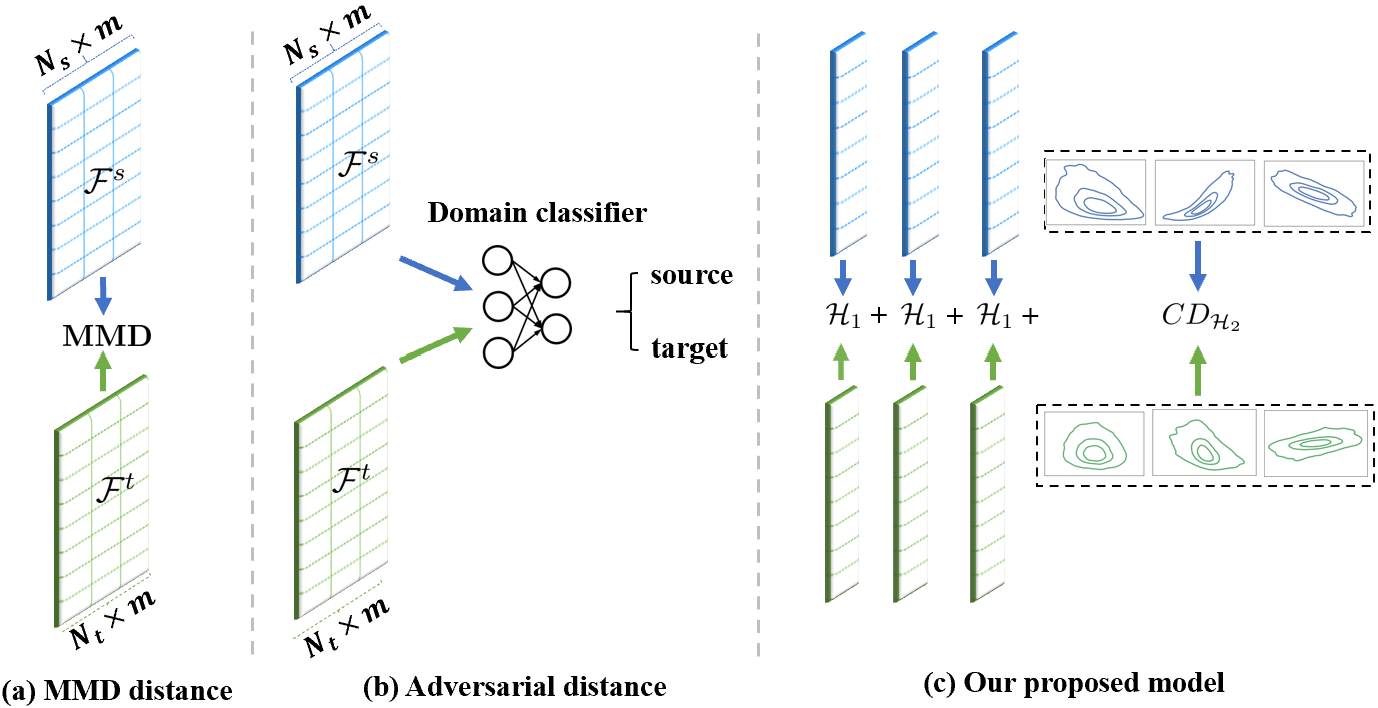}
\caption{There can be various choices of $\mathcal{H}$ to measure the feature difference in Eq. \eqref{EqDAN}, such as (a) MMD distance and (b) Wasserstein distance, JS distance, Proxy-$\mathcal{A}$ distance, etc. Our proposed model (c) looks into the detailed feature structures and regularizes the source error by dynamically adjusting the marginal divergence and the copula distance.}
\label{Fig_distance}
\end{center}
\end{figure*}

In summary, the detailed objective function to train a deep domain adaptation network is: 
\begin{equation}\label{EqDAN}
\min_{\theta_f,\theta_d}\, \big\{\mathcal{L}_{cls} (\theta_f, \theta_d)+\lambda \mathcal{H}(\mathcal{F}^s, \mathcal{F}^t) \big \},
\end{equation}
where the coefficient $\lambda$ controls the tradeoff between discriminative and transferrable feature learning.

\subsection{Copula-based domain adaptation networks}

In this work, we propose a copula-based domain adaptation network (CDAN) to allow more flexibility and emphasis on the dependence structure. To be specific, on the source domain, we split the distribution of the source features $\mathcal{F}^s \subseteq \mathbb{R}^{N_s \times m}$ into $m$ marginal distributions and the copula between the marginals. We do the same on the target domain. This facilitates us to evaluate the feature differences between the source domain and the target domain by the sum of two terms: (1) the sum of marginal feature differences $\{\mathcal{H} ( \mathcal{F}_i^s, \mathcal{F}_i^t)\}_{i=1}^m$, and (2) the copula distance between the source features and the target features $CD_{\mathcal{H}}( \mathcal{F}^s, \mathcal{F}^t;\bm{\beta})$. Here, for $* \in \{s,t\}$, $\mathcal{F}_i^*:= \{\mathbf{F}_{n,i}^*\}_{n=1}^{N_*}\subseteq \mathbb{R}^{N_* \times 1}$ with $\mathbf{F}_{n,i}^*$ being the $i$-th dimensional value of $\mathbf{F}_n^*$. In the later context, we call (1) the \textit{marginal divergence} (MD) and call (2) the \textit{copula distance} (CD). By adding the marginal divergence (with hyperparameters $\{\alpha_i\}_{i=1}^m$) and the copula distance (with hyperparameters $\bm{\beta}$) as regularization terms, we arrive at the objective function to train a CDAN, namely:
\begin{equation}\label{EqCDAN}
\min_{\theta_f,\theta_d}\, \big\{\mathcal{L}_{cls}(\theta_f, \theta_d) + \sum\limits_{i=1}\limits^m \alpha_i \mathcal{H}_1 (\mathcal{F}_i^s, \mathcal{F}_i^t)
+ CD_{\mathcal{H}_2}(\mathcal{F}^s, \mathcal{F}^t;\bm{\beta})\big \}.
\end{equation}

The detailed divergence framework is illustrated in Figure \ref{Fig_distance}(c). Notice that in Eq. \eqref{EqCDAN}, we differentiate the divergence metric $\mathcal{H}_1$ that is used to calculate the marginal divergence with $\mathcal{H}_2$ that is for the copula distance calculation. It is worth noting that our model can well accommodate the commonly-used divergence measures in literature.


Our proposed model has three advantages. On one hand, rather than encoding the divergence of each marginal feature and the dependence difference in a single value as in Eq. \eqref{EqDAN}, we split the divergence of the joint feature distributions. Thus we can identify to what extent the marginal feature differences and the copula distance contribute to the target risk respectively. Moreover, using hyperparameters $\{\alpha_i\}_{i=1}^m$ and $\bm{\beta}$ to separately control the marginal divergence and copula distance allows us to dynamically adjust the hyperparameters in a data-driven manner. It implicitly shows that there can be a tradeoff between the marginal divergence and the copula distance. Finally, different from using one distance metric to measure both the marginal feature differences and the dependence difference, our proposed model provides a more convenient and elaborate way to detect the changes of the marginal distributions and the dependence structure. 

It is straightforward to generalize our CDAN model to the regression tasks. Same as in the classification setting, a domain adaptation model for a regression task combines the feature extractor $F(\cdot;\theta_f)$ together with a regressor network $D(\cdot;\theta_d)$ to form the basis of a supervised learning network. We denote the predicted value for a sample $\mathbf{x}_n^s$ as $\widehat{y}_n^s:=D(F(\mathbf{x}_n^s;\theta_f);\theta_d)$. The regressor loss function is defined as the mean squared error between the predicted values $\{\widehat{y}_n^s\}_{n=1}^{N_s}$ and the ground-truth values $\{y_n^s\}_{n=1}^{N_s}$ on the source domain: $\mathcal{L}_{rgr} (\theta_f, \theta_d) := \sum_{n=1}^{N_s}(\widehat{y}_n^s-y_n^s)^2/N_s = \sum_{n=1}^{N_s}\big(D(F(\mathbf{x}_n^s;\theta_f);\theta_d)-y_n^s\big)^2/N_s$. Adding the marginal divergence and the copula distance as regularizer, we obtain the objective function to train a CDAN for a regression task:
\begin{equation*} 
\min_{\theta_f,\theta_d}\, \big\{\mathcal{L}_{rgr}(\theta_f, \theta_d) + \sum\limits_{i=1}\limits^m \alpha_i \mathcal{H}_1 (\mathcal{F}_i^s, \mathcal{F}_i^t) + CD_{\mathcal{H}_2}(\mathcal{F}^s, \mathcal{F}^t;\bm{\beta})\big \}.
\end{equation*}

\subsection{Algorithm}

The complete process of CDAN algorithm is presented in Algorithm 1. In particular, we provide a detailed description on how to learn the Gaussian copula parameter $\Sigma$ and how to update the model parameters $\theta_f$ and $\theta_d$.

\textbf{Learning the Gaussian copula parameter $\Sigma$.} \cite{ruppert2011statistics} proposes a moment-matching approach to learn the Gaussian copula parameter $\Sigma$ through Kendall's tau. Denote $\rho_{\tau}(X_i,X_j)$ as Kendall's tau between random variables $X_i$ and $X_j$, that is, $\rho_{\tau}(X_i,X_j):=\mathbb{E}[\text{sign}\big((X_i - \widetilde{X}_i)(X_j - \widetilde{X}_j)\big)]$, where $[\widetilde{X}_i,\widetilde{X}_j]$ is an independent copy of $[X_i,X_j]$. Then it can be proved that $\Sigma_{ij} = \sin \frac{\pi}{2} \rho_{\tau}(X_i, X_j)$. However, computing Kendall's tau by definition incurs a complexity of $O(N^2)$ when the sample size is $N$, making it expensive to train deep neural networks. Moreover, gradient vanishing occurs during training the neural network because of the \textit{sign} function. 


\IncMargin{1em}
\begin{algorithm} \SetKwData{Left}{left}\SetKwData{This}{this}\SetKwData{Up}{up} \SetKwFunction{Union}{Union}\SetKwFunction{FindCompress}{FindCompress} \SetKwInOut{Input}{Input}\SetKwInOut{Output}{Output}
	
	\Input{Source data $\mathcal{D}_s = \{(\mathbf{x}_n^s,y_n^s)\}_{n=1}^{N_s}$, Target data $\mathcal{D}_t = \{\mathbf{x}_n^t\}_{n=1}^{N_t}$, Maximum training epoch $S$, Divergence metric $\mathcal{H}_1$, $\mathcal{H}_2$, Parameters $\{\alpha_i\}_{i=1}^m$, $\bm{\beta}$.}
	\Output{Optimal model parameters $\theta_f,\theta_d$.}
	 \BlankLine 
	 
	 \For{epoch $=1$ to $S$}{ 
	 	$\mathcal{F}^s \leftarrow \{F(\mathbf{x}_n^s;\theta_f)\}_{n=1}^{N_s}$;
	 	
	 	$\mathcal{F}^t \leftarrow  \{F(\mathbf{x}_n^t;\theta_f)\}_{n=1}^{N_t}$;

	 	$MD(\theta_f) \leftarrow  \sum\limits_{i=1}\limits^m \alpha_i \mathcal{H}_1(\mathcal{F}^s_i, \mathcal{F}^t_i)$;
	 	
	 	$CD(\theta_f) \leftarrow CD_{\mathcal{H}_2}(\mathcal{F}^s,\mathcal{F}^t;\bm{\beta})$ as calculated in Eq. \eqref{EqCD};
	 	
	 	Get the source error $\mathcal{L}_{cls}(\theta_f,\theta_d)$ with Eq. \eqref{eq_src_error};
	 	
	 	$Loss \leftarrow \mathcal{L}_{cls}(\theta_f,\theta_d) + MD(\theta_f) +CD(\theta_f)$;
	 	
	 	$Loss.backward()$
	 	
 	 } 
 	 	  \caption{CDAN algorithm}
 	 	  \label{algo_disjdecomp} 
 	 \end{algorithm}
 \DecMargin{1em} 

To address the gradient vanishing issue, we propose to replace the $sign$ function with the $\tanh$ function with parameter $a$, namely, $\rho_{\tau}(X_i,X_j)\approx\rho(X_i,X_j;a):=\mathbb{E}[\tanh\big(a(X_i - \widetilde{X}_i)(X_j- \widetilde{X}_j)\big)]$. We prove that $\text{lim}_{a \to \infty} \rho(X_i,X_j;a) = \rho_\tau (X_i,X_j)$ (with the proof attached in Appendix.) 

\begin{proposition}\label{prop_Gaussian}
For two random variables $X_1$ and $X_2$, it holds that $\lim \limits_{a \to \infty} \rho(X_1, X_2;a) = \rho_\tau(X_1,X_2)$.
\end{proposition}

To further reduce the computational complexity, we adopt the unbiased estimate of Kendall's tau that can be computed with linear complexity. More specifically, given $\{[x_{n,1},x_{n,2}]\}_{n=1}^N \subseteq \mathbb{R}^{N\times 2}$ as $N$ realizations of $[X_i,X_j]$, an unbiased estimator for $\rho_{\tau}(X_i,X_j)$ is $\rho_{\tau}(X_i,X_j) = \frac{2}{N} \sum_{n=1}^{N/2} \tanh \big(a(x_{2n-1,1}-x_{2n,1})(x_{2n-1,2} - x_{2n,2})\big)$. It significantly reduces the computational cost of Kendall's tau from $O(N^2)$ to $O(N)$.

\textbf{Learning the model parameters $\theta_f$ and $\theta_d$.} For illustration convenience, we conduct the calculation in the classification setting. It can be similarly done for the regression problem, thus we omit here. From Eq. \eqref{EqCDAN}, we have:
\begin{equation*}
\begin{aligned}
\bigtriangledown_{\theta_d} &= \partial_{\theta_d} \mathcal{L}_{cls}(\theta_f,\theta_d) \\
&= -  \sum \limits_{n=1}^{N_s} \sum \limits_{i=1}^l \mathbf{1}(y_n^s = i)\cdot  \partial_{\theta_d} D(\mathbf{F}_n^s;\theta_d)_i / \big(N_s D(\mathbf{F}_n^s;\theta_d)_i \big),
\end{aligned}
\end{equation*}
\begin{equation*}
\begin{aligned}
\bigtriangledown_{\theta_f} = \sum \limits_{i=1}^m \alpha_i \partial_{\theta_f} \mathcal{H}_1(\mathcal{F}_i^s, \mathcal{F}_i^t)+ \partial_{\theta_f} CD_{\mathcal{H}_2} + \partial_{\theta_f} \mathcal{L}_{cls}(\theta_f,\theta_d),
\end{aligned}
\end{equation*}
where $CD_{\mathcal{H}_2}$ is the abbreviation for $CD_{\mathcal{H}_2}(\mathcal{F}^s, \mathcal{F}^t;\bm{\beta})$.

	
	 
	 	

	 	
	 	
	 	
	 	

By the chain rule,
\begin{equation*}
\begin{aligned}
&\partial_{\theta_f} \mathcal{L}_{cls}(\theta_f,\theta_d) \\
=& -\sum\limits_{n=1}^{N_s} \sum\limits_{i=1}^l \frac{\mathbf{1}(y_n^s = i)\cdot \partial_{\mathbf{F}_n^s} D(\mathbf{F}_n^s;\theta_d)_i \cdot \partial_{\theta_f} F(\mathbf{x}_n^s; \theta_f)}{N_s D(\mathbf{F}_n^s;\theta_d)_i }.
\end{aligned}
\end{equation*}

The derivatives $\partial_{\theta_f} \mathcal{H}_1(\mathcal{F}_i^s, \mathcal{F}_i^t)$ and $\partial_{\theta_f} CD_{\mathcal{H}_2}(\mathcal{F}^s, \mathcal{F}^t;\bm{\beta})$ depend on the choice of $\mathcal{H}_1$ and $\mathcal{H}_2$. In the experiment, we will take $\mathcal{H}_1$ as the MMD distance and $\mathcal{H}_2$ as the KL divergence. Specifically, if $\mathcal{H}_1$ is MMD distance with the characteristic kernel function $k$, the unbiased estimate of squared MMD distance between $\mathcal{F}_i^s$ and $\mathcal{F}_i^t$ is given as \cite{long2015learning}: 
\begin{equation*}
\begin{aligned}
\mathcal{H}_{\text{MMD}}^2 (\mathcal{F}_i^s, \mathcal{F}_i^t) := &\sum_{n,n'=1}^{N_s} \frac{k(\mathbf{F}_{n,i}^s,\mathbf{F}_{n',i}^s)}{N_s^2} +  \sum_{n,n'=1}^{N_t} \frac{k(\mathbf{F}_{n,i}^t,\mathbf{F}_{n',i}^t)}{N_t^2}\\
&- \sum_{n=1}^{N_s}\sum_{n'=1}^{N_t} \frac{2k(\mathbf{F}_{n,i}^s,\mathbf{F}_{n',i}^t)}{N_sN_t}.
\end{aligned} 
\end{equation*}
Thus, 
\begin{equation*}
\begin{aligned}
&\partial_{\theta_f} \mathcal{H}_{\textup{MMD}}(\mathcal{F}_i^s, \mathcal{F}_i^t) \\
= &\Big[\sum\limits_{n,n'=1}^{N_s} \frac{\partial_{\theta_f} k(\mathbf{F}_{n,i}^s,\mathbf{F}_{n',i}^s)}{N_s^2}  +\sum\limits_{n,n'=1}^{N_t} \frac{\partial_{\theta_f} k(\mathbf{F}_{n,i}^t,\mathbf{F}_{n',i}^t)}{N_t^2}\\
&- \sum\limits_{n=1}^{N_s}\sum\limits_{n'=1}^{N_t} \frac{2\partial_{\theta_f} k(\mathbf{F}_{n,i}^s,\mathbf{F}_{n',i}^t)}{N_sN_t}\Big]/\big(2\mathcal{H}_{\textup{MMD}}(\mathcal{F}^s, \mathcal{F}^t)\big).
\end{aligned}
\end{equation*}
When $\mathcal{H}_2$ is the KL divergence, then

\begin{equation*}
\begin{aligned}
\partial_{\theta_f} CD_{\mathcal{H}_{KL}} 
= \sum \limits_{ i < j } \frac{\beta_{ij} \partial_{\theta_f} |\log(1-(\Sigma_{ij}^s)^2) /(1-(\Sigma_{ij}^t)^2)|}{2}, 
\end{aligned}
\end{equation*}
where $\Sigma_{ij}^s$ ($\Sigma_{ij}^t$) is the Gaussian copula parameter of $\{[\mathbf{F}^s_{n,i},\mathbf{F}^s_{n,j}]\}_{n=1}^{N_s}$ ($\{[\mathbf{F}^t_{n,i},\mathbf{F}^t_{n,j}]\}_{n=1}^{N_t}$).

\section{Experiments}\label{SecExp}

To prove the efficacy of our proposed model, we test it on one toy dataset and three real-world datasets, with the details of the three real-world datasets summarized in Table \ref{Tb_datasets}. All experiments in this section are run on Dell 7920 with Intel(R) Xeon(R) Gold 6250 CPU at 3.90GHz, and a set of NVIDIA Quadro RTX 6000 GPU. The code is available at \href{https://github.com/yzR1991/Deep_into_The_Domain_Shift}{https://github.com/yzR1991/Deep{\textunderscore}into{\textunderscore}The{\textunderscore}Domain{\textunderscore}Shift}. We run our python code on the anaconda virtual environment with Microsoft Windows Server 2019 Standard as OS. The code can also be run in other OS, device, or environment, with Pytorch version 1.7 or above.

\begin{table}[!htbp]
\footnotesize
\centering
\caption{Information of the real-world datasets used in this paper.}
\begin{tabular}{cccc} 
\toprule
\textbf{Dataset} &  $\sharp$\textbf{Instances} & $\sharp$\textbf{Features} & \textbf{Task} \\
\midrule
Retail credit data  & 1100000 & 69 & Classification \\
Equity price data  & 71242 & 22 & Regression \\
UCI wine quality  & 6197 & 12 & Regression \\
\bottomrule
\end{tabular}
\label{Tb_datasets}
\end{table}

\subsection{Toy problem: Two inter-twinning moons}

The source domain considered here is the classical binary problem with two inter-twinning moons, each class corresponding to one moon. Specifically, for the blue moon in Figure \ref{Fig_synthetic}, the points roughly falls on the upper half circle of $y = \sqrt{1-x^2}$, with points for the red moon falling on lower half circle of $y = 0.5-\sqrt{1-(1-x)^2}$. We then consider 4 different target domains by stretching the circle into ellipses where the length of the major axis can be 2, 3, 4 and 5 times that of the minor axis. Such stretching from source domain to target domain strongly affects the relationship between the vertical ordinates and the horizontal ordinates of each point, thus making changes to the internal dependence structure of the data. For each domain, we generate 1024 instances (512 of each class). 

For each transfer task, we compare the CDAN model with DAN \cite{long2015learning}, CORAL\cite{sun2016return} and the no-adaptation baseline (MLP). Each trial of the models is repeated 10 times, and we report the average accuracy in Table \ref{Tb_synthetic}. We remark that the larger the length of the major axis, the more difficult the problem becomes, as all of the four models unanimously show a weaker adaptation ability. Our CDAN provides the best performance in all of the four transfer tasks, indicating that CDAN actually captures the dependence difference precisely.

\begin{figure*}[!htbp]
\centering
 \includegraphics[width=0.24\textwidth]{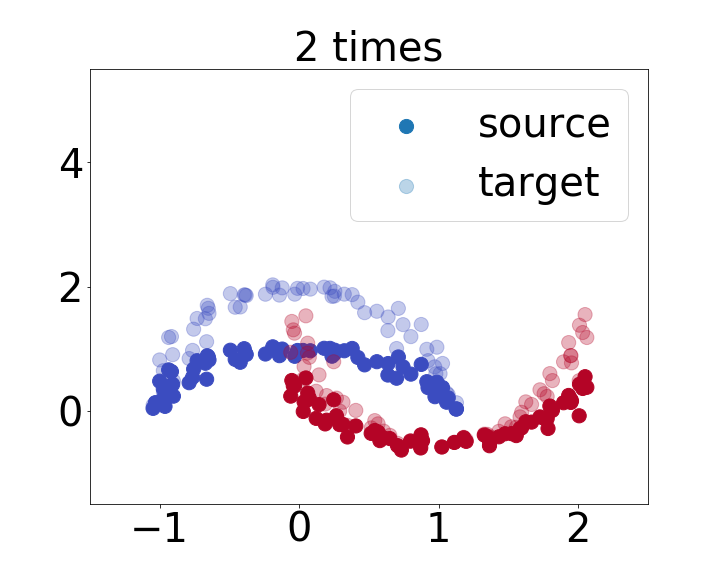}
 \includegraphics[width=0.24\textwidth]{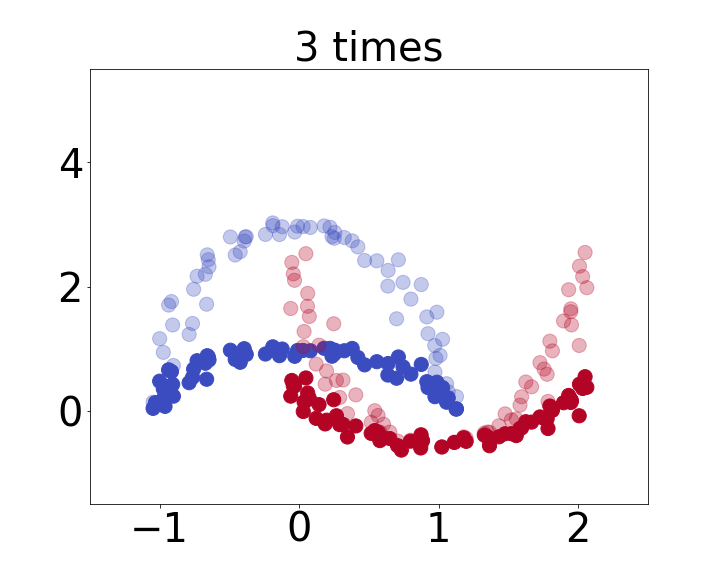}
 \includegraphics[width=0.24\textwidth]{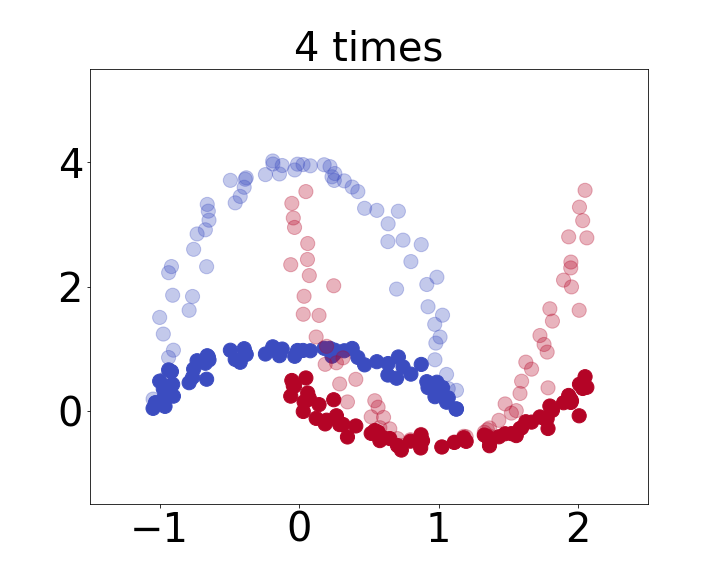}
 \includegraphics[width=0.24\textwidth]{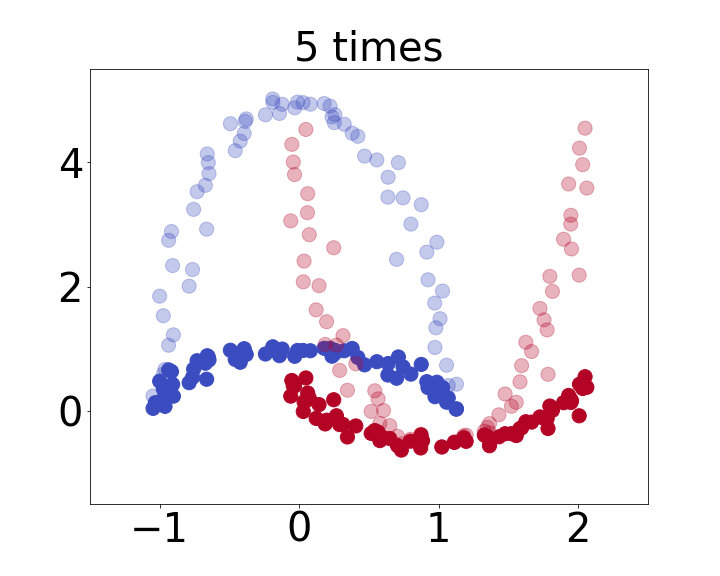}
 \caption{Illustration of the four transfer tasks on the synthetic dataset. The two
classes of the source samples are blue and red, and points that are more transparent represent the target samples.}
\label{Fig_synthetic}
\end{figure*}

\begin{table}[!htpb] 
\caption{Accuracy for the inter-twinning moons dataset.}
\label{Tb_synthetic}
\centering
\begin{tabular}{lcccc}
\toprule
  & 2 times & 3 times & 4 times & 5 times      \\
\midrule
MLP & 88.99 & 80.88 & 78.73 & 77.10 \\
CORAL & 97.45 & 94.40 & 92.32 & 90.26 \\
DAN & 97.04 & 94.17 & 91.18 & 91.04 \\
CDAN & \textbf{97.91} & \textbf{94.42} & \textbf{93.17} & \textbf{91.54}\\

\bottomrule
\end{tabular}
\end{table}

\subsection{Retail credit classification}\label{Sec_credit}

In this section, we apply our methods to improve the classification accuracy of a credit risk model on a novel real-world anonymous dataset. The dataset is kindly provided by one of the largest global technology firms that operates in both the e-commerce business and the lending business. It records the monthly credit status of roughly one-half million customers, spanning from 2016 January to 2020 June. The credit status shows whether a customer is in default. In addition to the credit status, customers' monthly shopping, purchasing, and loan history are also included in the dataset in detail. We collect 69 features for each customer in each month from the raw dataset and normalize them to the range $[0,1]$. A binary classification model is constructed to distinguish customers who default (labeled as 0) from those who have paid off all the debts on time (labeled as 1).

Domain adaptation is needed when one forecasts customers' credit risk with the classification model trained with the past data because significant distribution shifts exist between months, especially between the off-season and the peak season, between before-COVID times and after-COVID times. We record the distribution shifts between two consecutive months in Figure \ref{Fig_creditshift}. Specifically, we collect 2 features that show great importance for the classification: each customer's monthly total purchase and his monthly credit ratio (available credit amount / total credit limit). To illustrate the distribution shifts, we take the three points on 19Jun in Figure \ref{Fig_creditshift} as an example. The blue (blue-dashed, resp.) point records the MMD distance of monthly purchase (credit ratio, resp.) distribution between May and June, and the black point records the copula distance between May and June. From Figure \ref{Fig_creditshift}, we can identify three peak periods circled by a red box (19Jun, 19Nov, 20Feb) that show substantial marginal differences, verifying the necessity of transfer learning. What's more, the copula distance remains high over the whole year, suggesting that special attention to the dependence difference is in urgent need.
\begin{figure}[!htbp]
\begin{center}
\includegraphics[width=0.9\columnwidth]{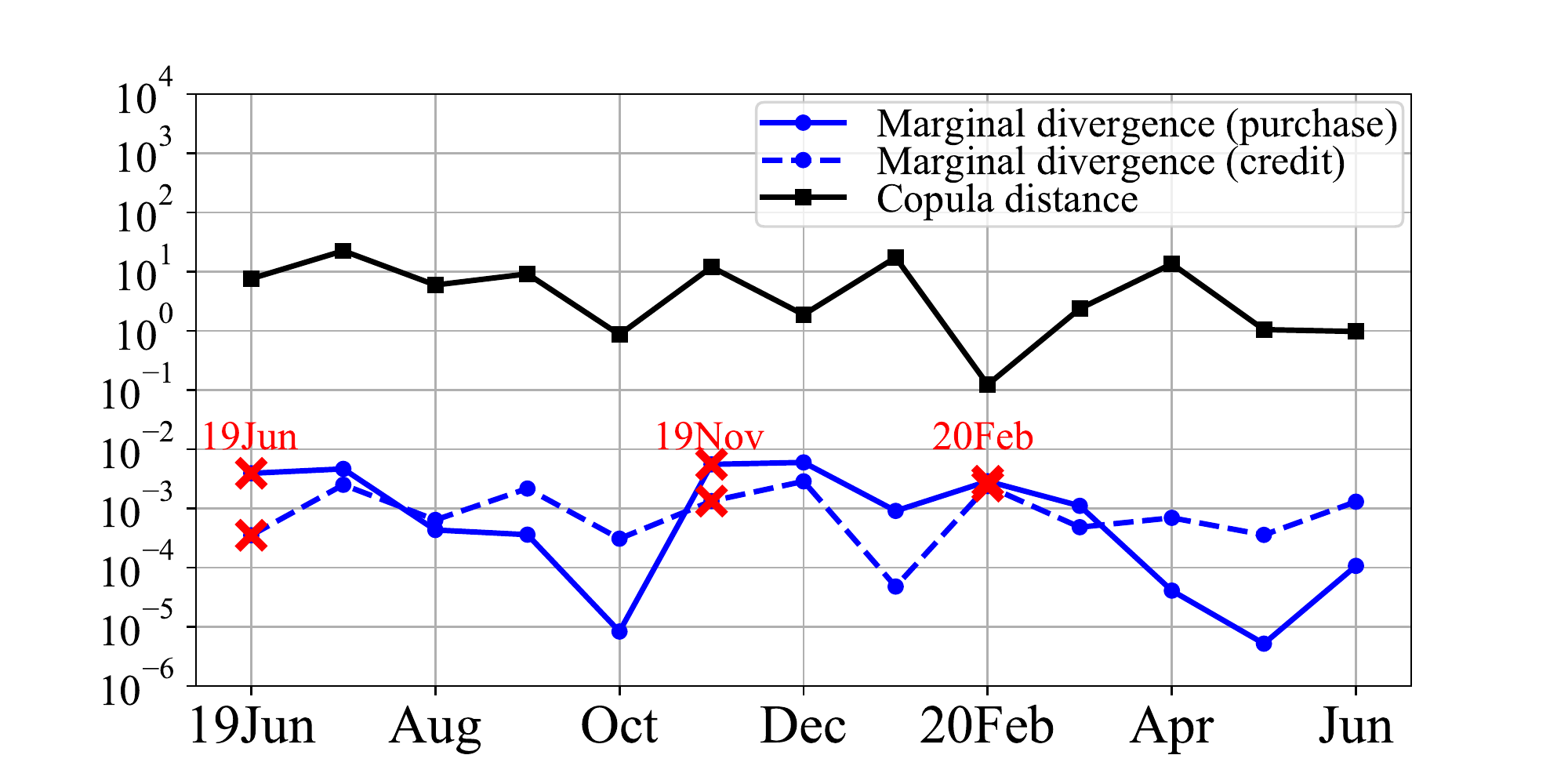}
\caption{Distribution shift in the raw data.}
\label{Fig_creditshift}
\end{center}
\end{figure}

\begin{table*}[h]
  \caption{Performance of retail credit classification in 8 transfer tasks.}
  \label{Tb_credit}
  \centering
  \footnotesize
  \begin{tabular}{l c c c c c c c c}
  \toprule
    & 19May$\rightarrow$Jun & 19Oct$\rightarrow$Nov & 19Dec$\rightarrow$20Jan & 20Jan$\rightarrow$Feb & 20Feb$\rightarrow$Mar & 20Mar$\rightarrow$Apr & 20Apr$\rightarrow$May & 20May$\rightarrow$Jun \\
  \midrule
  MLP   &       73.85 $\pm$ 0.05 &   56.11 $\pm$ 0.05 & 90.66 $\pm$ 0.005 &   75.60 $\pm$ 0.03 &  76.55 $\pm$ 0.04 &  90.78 $\pm$ 0.009 &  91.61 $\pm$ 0.005 &  88.06 $\pm$ 0.018\\
  DAN   &       79.00 $\pm$ 0.03 &    79.31 $\pm$ 0.05 &   92.09 $\pm$ 0.004 &    87.41 $\pm$ 0.01 &   91.04 $\pm$ 0.01 & 91.88 $\pm$ 0.008 &  92.78 $\pm$ 0.002 &  90.88 $\pm$ 0.009 \\
  CORAL &    77.86 $\pm$ 0.04 &   76.23 $\pm$ 0.05 &    \textbf{92.12 $\pm$ 0.003} &   87.13 $\pm$ 0.01 &   89.15 $\pm$ 0.01 &  91.94 $\pm$ 0.008 &  \textbf{92.79 $\pm$ 0.003} &  90.94 $\pm$ 0.009 \\
  AFN & 77.53 $\pm$ 0.02 & 80.01 $\pm$ 0.02 & 83.73 $\pm$ 0.02 & 82.72 $\pm$ 0.02 & 83.32 $\pm$ 0.01 & 84.52 $\pm$ 0.01 & 86.24 $\pm$ 0.01 & 84.71 $\pm$ 0.02 \\
  MCD & 75.30 $\pm$ 0.10 & 71.08 $\pm$ 0.11 & 82.60 $\pm$ 0.08 & 78.40 $\pm$ 0.09 & 76.94 $\pm$ 0.08 & 86.48 $\pm$ 0.04 & 86.00 $\pm$ 0.04 & 84.85 $\pm$ 0.05 \\
  CDAN  &       \textbf{80.44 $\pm$ 0.04} &  \textbf{80.79 $\pm$ 0.06} &       91.93 $\pm$ 0.005 &  \textbf{88.20 $\pm$ 0.02} &  \textbf{91.51 $\pm$ 0.01} &       \textbf{92.46 $\pm$ 0.010} &       92.74 $\pm$ 0.003 &       \textbf{91.39 $\pm$ 0.010} \\
  \bottomrule
  \end{tabular}
  \end{table*}
We first evaluate our methods on the transfer between the off-seasons and the peak seasons (specifically, the sales seasons June and November), and build two transfer tasks: 19May$\rightarrow$Jun, 19Oct$\rightarrow$Nov. We further investigate the COVID impact on transferring the classification models and include the evaluation on 6 more transfer tasks: 19Dec$\rightarrow$20Jan, 20Jan$\rightarrow$Feb, 20Feb$\rightarrow$Mar, 20Mar$\rightarrow$Apr, 20Apr$\rightarrow$May, 20May$\rightarrow$Jun. In each transfer task, the source domain consists of samples from the former month and the target domain samples are from the latter month. The sample size for each domain is in the magnitude of $10^5$ customers. 

We mainly follow standard evaluation protocol \cite{xiao2021implicit} for unsupervised domain adaptation and use all source samples with binary labels and all target samples without labels \cite{long2015learning}. We compare our CDAN model to 4 classical domain adaptation models: DAN \cite{long2015learning}, CORAL \cite{sun2016return}, AFN \cite{afn}, MCD \cite{mcd} as well as the no-adaptation baseline (MLP), which is a fully connected neural network with multiple hidden layers. For our model CDAN, we set $\mathcal{H}_1$ to be the MMD distance \cite{long2015learning}, and set $\mathcal{H}_2$ to be the KL divergence \cite{Belov2011}. We set $\alpha_i = \alpha\,(\forall\, i)$ and $\beta_{ij} = \beta\,(\forall\,i,j)$, and select the hyperparameter pair $(\alpha,\beta)$ by grid search. The detailed implementation procedures are summarized in the Appendix.

In Table \ref{Tb_credit}, we record the averages and standard errors of AUC over 100 randomized trials for each model in each task. CDAN outperforms the other models in 6 transfer tasks. It deserves attention that the outperformance is significant in that the increase in the AUC score is far larger than the standard deviation. As an illustration, we plot the density of the learned features' MD and CD in Figure \ref{Fig_MDCD} for each model. The density is estimated from the 100 trials for each model in a specific transfer task (20Jan$\rightarrow$Feb). We find that CDAN again shows superiority in terms of contracting both the marginal and the dependence differences. Such observation, together with the optimized hyperparameters, explains why CDAN outperforms other domain adaptation models.
\begin{figure}[!htbp]
\centering
 \includegraphics[width=0.45\columnwidth]{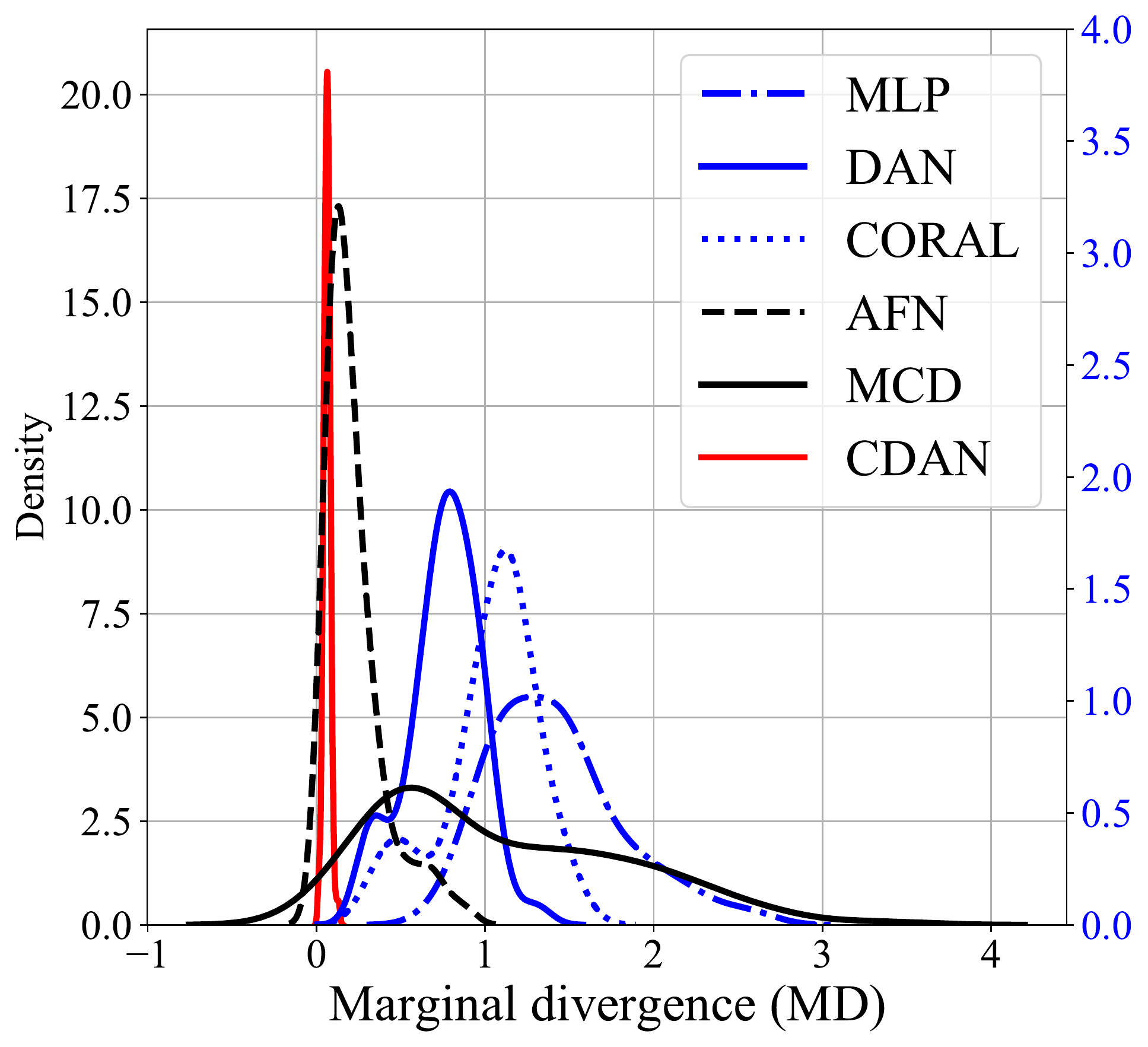}
 \includegraphics[width=0.45\columnwidth]{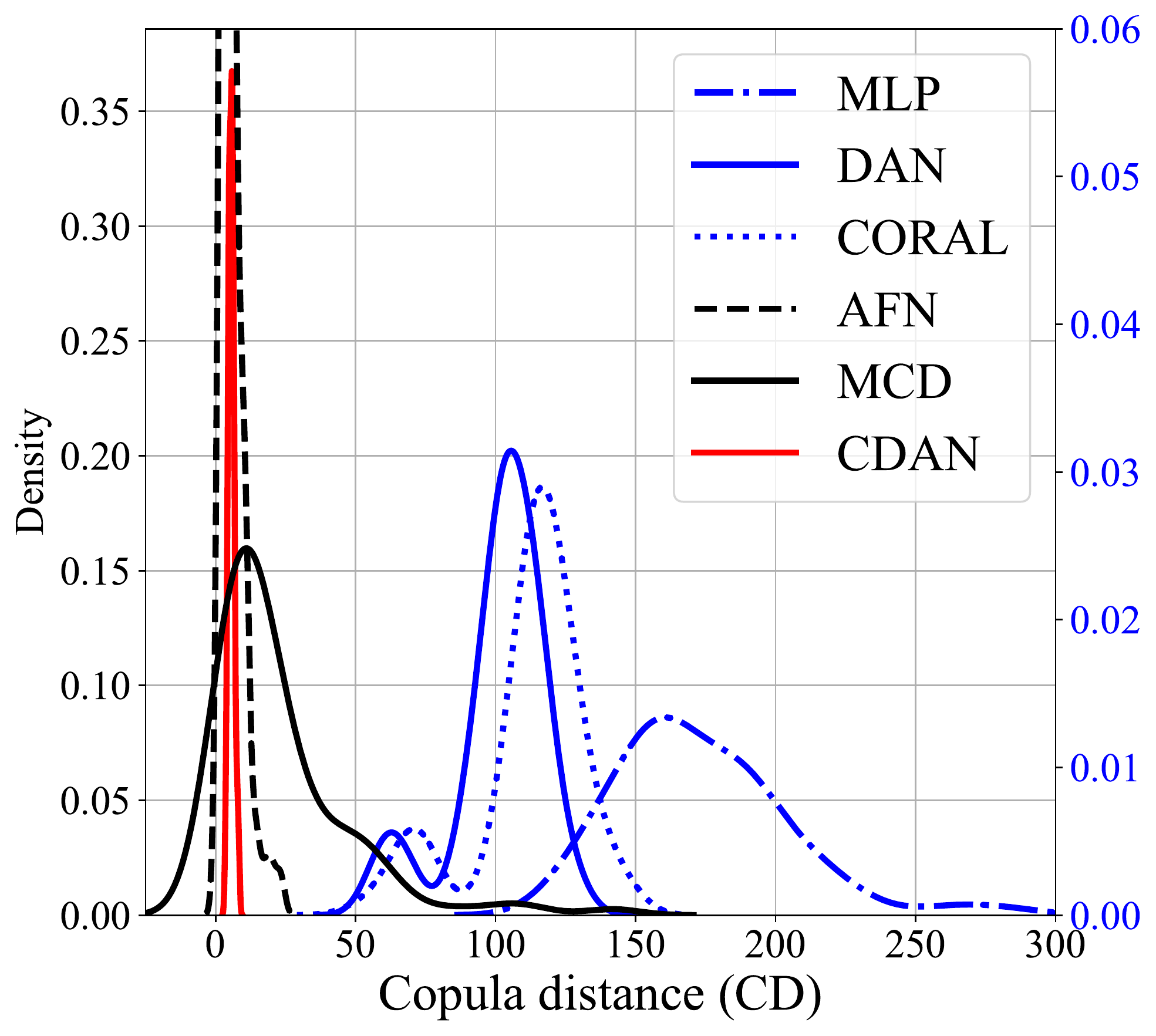}
 \caption{Distribution of the learned features' marginal divergence (left) and copula distance (right) over 100 trials for each model.}
 \label{Fig_MDCD}
\end{figure}

\subsection{Intra-day equity price regression}

We collect the intraday 5-minute asset prices of 22 stocks selected from HKEX according to the market cap and daily turnover. The 22 stocks cover 8 industries (according to the Hang Seng Industry Classification System) that include Information Technology, Financials, Consumer Discretionary, etc. The data spans from Dec 1st, 2014 to Dec 31st, 2020, and consists of 71242 observations with information of the first half-hour in each trading day excluded. We divide the observations into two domains according to the Hang Seng Index (HSI) daily returns. Specifically, the target domain includes observations of the days when the daily return is less than the 0.1-quantile of the whole 6-year return series, and the source domain includes the remaining observations. The goal is to forecast the next 5-minute price of the 22 stocks with their last-hour price as the input. 

\begin{table}[!htbp]
\small
\centering
\caption{The mean $\alpha$-quantiles $Q^s_\alpha$ ($Q^t_\alpha$) over all 22 stocks and the mean quantile dependence $\tau^s_\alpha$ ($\tau^t_\alpha$) over all stock pairs in the source (target) domain. Values in the brackets record the corresponding minimum and maximum.}
\begin{tabular}{cccc} 
\toprule
 & $\alpha=0.03$          & $\alpha=0.10$                                     & $\alpha = 0.25$                                        \\ 
\midrule
           $Q^s_\alpha$            & \begin{tabular}[c]{@{}c@{}}78.5\\(37.7, 114.8)\end{tabular}   & \begin{tabular}[c]{@{}c@{}}87.4\\(39.8, 147.6)\end{tabular}  & \begin{tabular}[c]{@{}c@{}}103.9\\(44.3, 198.3)\end{tabular}  \\ 
           $Q^t_\alpha$            & \begin{tabular}[c]{@{}c@{}}76.6\\(36.7, 131.3)\end{tabular}   & \begin{tabular}[c]{@{}c@{}}84.2\\(38.7, 149.0)\end{tabular}  & \begin{tabular}[c]{@{}c@{}}97.7\\(43.5, 158.7)\end{tabular}   \\ 
\midrule
            $\tau^s_\alpha$           & \begin{tabular}[c]{@{}c@{}}28.9\\(4.3e-2, 88.0)\end{tabular}  & \begin{tabular}[c]{@{}c@{}}36.8\\(1.6e-1, 89.6)\end{tabular} & \begin{tabular}[c]{@{}c@{}}48.5\\(8.4e-1, 98.6)\end{tabular}  \\ 
             $\tau^t_\alpha$          & \begin{tabular}[c]{@{}c@{}}42.6\\(1.7e-1, 100.0)\end{tabular} & \begin{tabular}[c]{@{}c@{}}47.0\\(2.2, 93.8)\end{tabular}    & \begin{tabular}[c]{@{}c@{}}58.4\\(2.4, 99.7)\end{tabular}     \\
\bottomrule
\end{tabular}
\label{Tb_comp2}
\end{table}

\begin{table*}[h]
  \caption{Efficiency of equity price forecast for the 6 models. Quantities are averaged over 100 runs.}
  \label{model_comparison}
  \centering
  \small
  \begin{tabular}{lccccccc}
    \toprule
             & RMSE     & RE(\%) & Q1 RE(\%) & Q2 RE(\%) & Q3 RE(\%) & LRE(\%) & SRE(\%)  \\
    \midrule
    RNN & 0.0507 $\pm$ 3.4e-3 & 3.33 $\pm$ 0.22 & 3.23 & 3.35 & 3.47 & 9.15 $\pm$ 0.90 & 1.02 $\pm$ 0.12\\
    LSTM    & 0.0446 $\pm$ 1.5e-3  & 3.04 $\pm$ 0.12 & 2.96  & 3.02 & 3.12 & 8.62 $\pm$ 0.70 & 0.94 $\pm$ 0.10\\
    DANN    & 0.0475 $\pm$ 9.6e-3  & 3.00 $\pm$ 0.56 & 2.59 & 2.90 & 3.32 & 8.44 $\pm$ 1.60  & 0.98 $\pm$ 0.22\\ 
    CORAL & 0.0449 $\pm$ 1.5e-3 & 3.06 $\pm$ 0.14   & 2.97 & 3.06  & 3.15 & 8.68 $\pm$ 0.87 &  0.96 $\pm$ 0.09 \\
    DAN     & 0.0255 $\pm$ 2.5e-3 & 1.89 $\pm$ 0.21 & 1.75 & 1.89 & 2.05 & 6.44 $\pm$ 1.55  & 0.45 $\pm$ 0.10\\
    CDAN    & \textbf{0.0235 $\pm$ 2.3e-3} & \textbf{1.77 $\pm$ 0.21} & \textbf{1.63} & \textbf{1.75} & \textbf{1.89} & \textbf{6.06 $\pm$ 1.53}  & \textbf{0.43 $\pm$ 0.09}\\
    \bottomrule
  \end{tabular}
\end{table*}
To illustrate the distribution shift of the two domains, we record the quantiles and the quantile dependence of each domain in Table \ref{Tb_comp2}. We denote $X_i^s$ ($X_i^t$) as the price series for stock $i$ in the source (target) domain and define its whole price series as $X_i := X_i^s \bigcup X_i^t$. 
Each $X_i$ is normalized by its first price. For a given quantile level $\alpha$ and $*\in\{s,t\}$ indicating the domain, the average $\alpha$-quantile of domain $*$ is defined by $Q_{\alpha}^*:=\sum_{i=1}^{22} Q_{\alpha,i}^*/22$, where $Q_{\alpha,i}^*$ is the $\alpha$-quantile of the price series$X_i^*$. The quantile dependence is given by $\tau_{\alpha}^*:=\sum_{1\leq i < j \leq 22} \tau_{\alpha,[i,j]}^*/231$, where $\tau_{\alpha,[i,j]}^*:= 100 \times Pr(X_j^* \leq Q_{\alpha,j}|X_i^* \leq Q_{\alpha,i})$ and $Q_{\alpha,i}$ is the $\alpha$-quantile of $X_i$. We see that the marginal quantiles of either domain do not differ that much, but as $\alpha$ increases, the marginals' differences between the two domains get larger. Furthermore, the quantile dependence significantly differs irrespective of the $\alpha$ value.

We compare to the following 5 models: RNN, LSTM, DANN \cite{ganin2016domain}, CORAL \cite{sun2016deep} and DAN \cite{long2015learning}. Specifically, RNN and LSTM serve as the no-adaptation benchmarks and only utilize the source samples to do the training. CORAL and DAN are the same as in Section \ref{Sec_credit}, except that they use an LSTM as a feature extractor in this section. Though AFN and MCD are SOTA models, they are not originally designed for sequential data and perform not well in this case, so we do not list the corresponding results. For more implementation details, see the Appendix.
\begin{figure}[!htbp]
\centering
 \subfloat[\footnotesize{Comparison of test RMSE (left) and RE distribution (right).}]{
  \includegraphics[width=0.45\columnwidth]{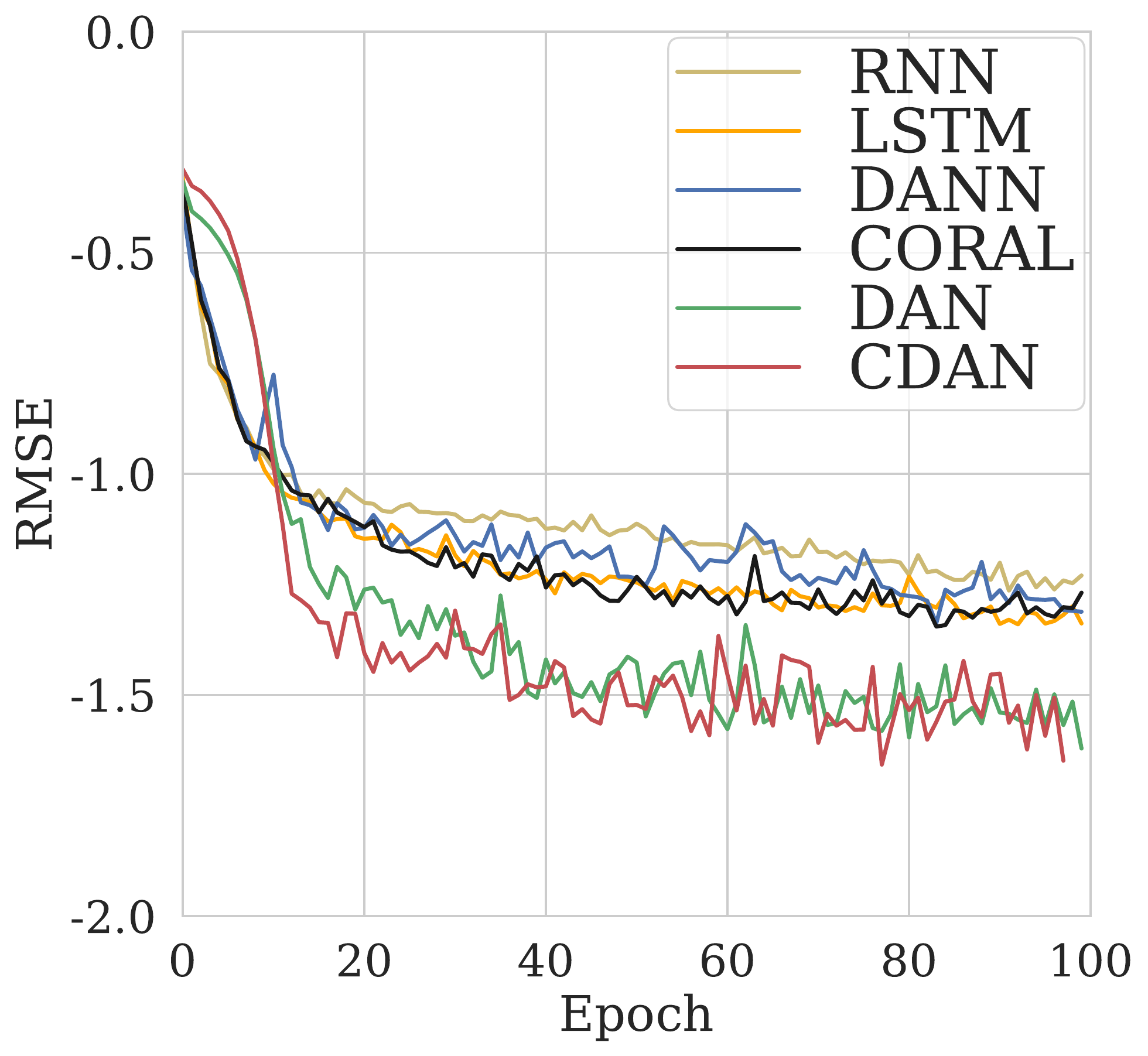}
  \includegraphics[width=0.45\columnwidth]{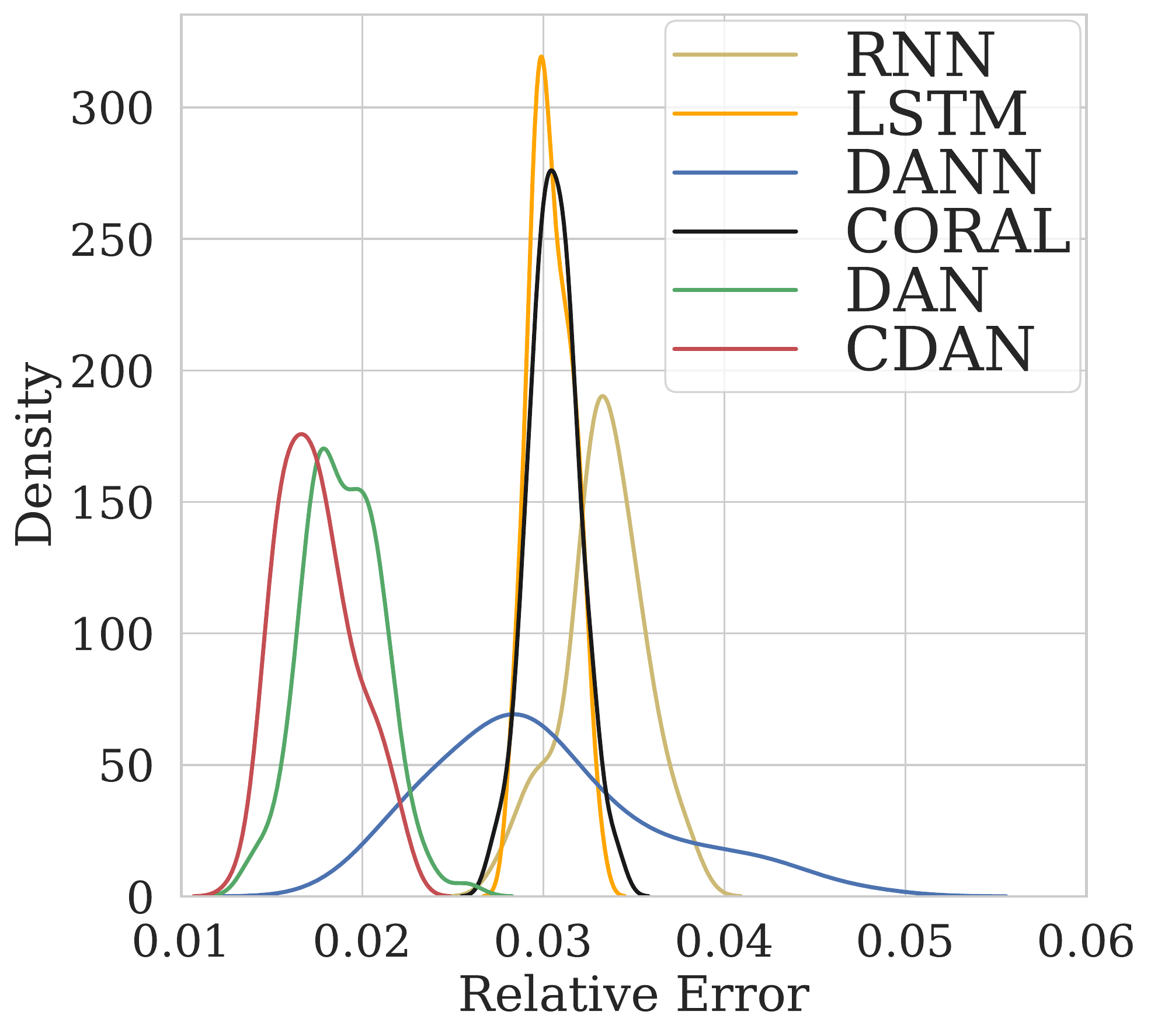}
 }
 \hfill
 \subfloat[\footnotesize{DAN model (left) v.s. CDAN model (right).}]{
  \includegraphics[width=0.45\columnwidth]{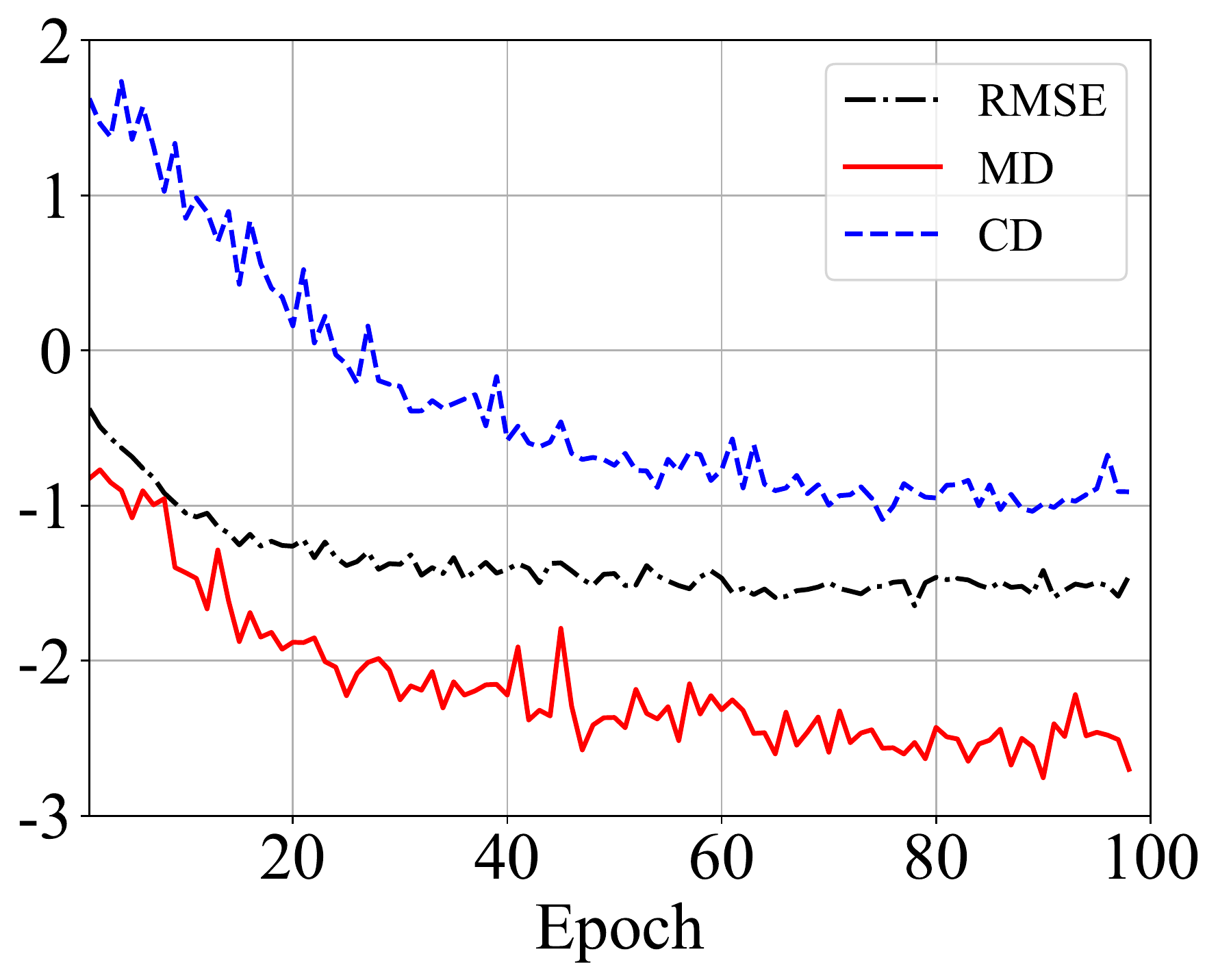}
  \includegraphics[width=0.45\columnwidth]{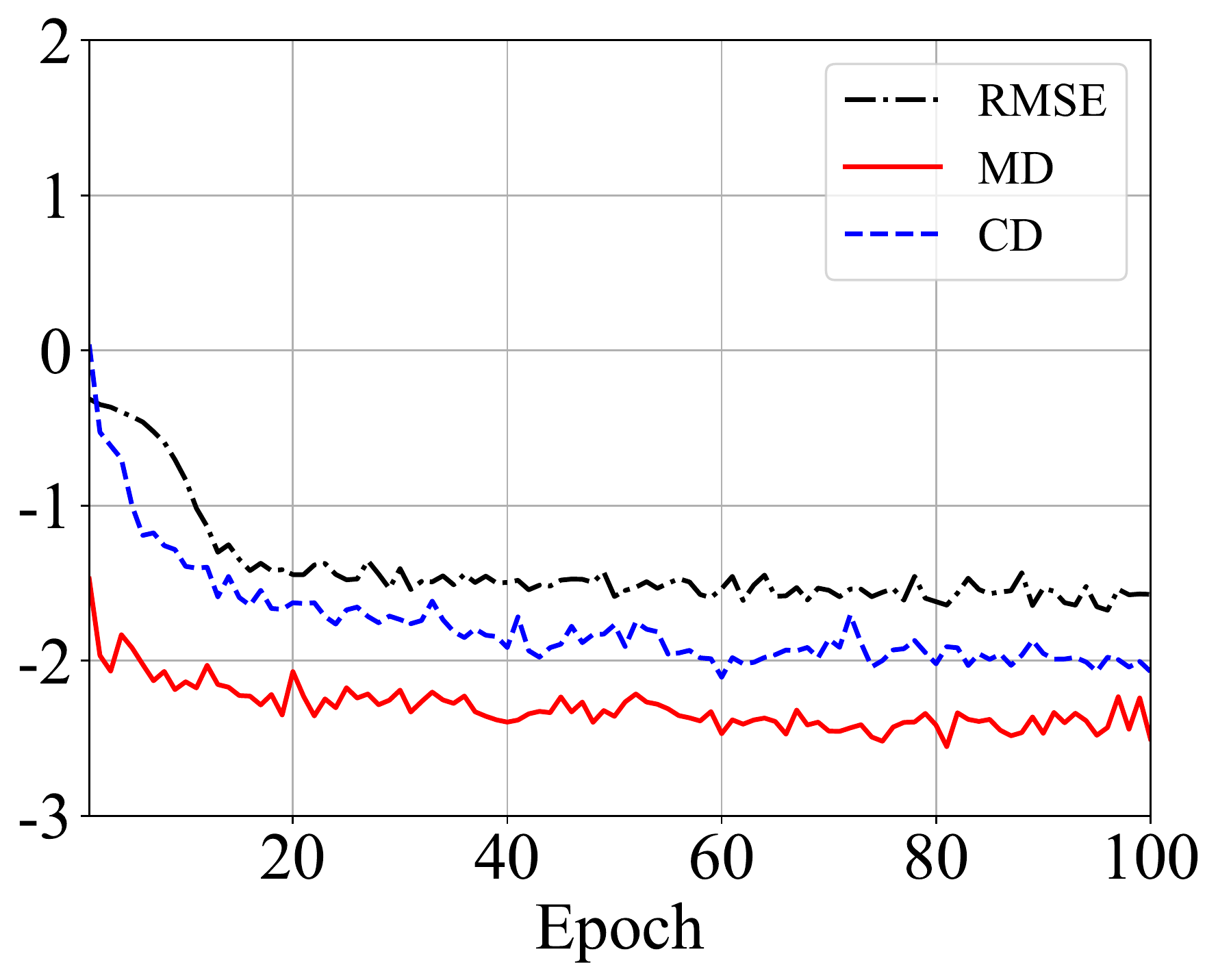}
 }
 \caption{The performance of forecasting the equity price in a particular (typical) trial. All values are in the base 10 logarithms.}
 \label{Fig_rmse_md_cd}
\end{figure}

In Table \ref{model_comparison}, we summarize the experimental results of the 6 models. There are 7 performance metric columns in Table \ref{model_comparison}. The 1st column represents the mean and the standard deviation of the RMSE over 100 trials. The 2nd-5th columns record the detailed information of the relative errors (RE) over 100 trials. Specifically, the 2nd column records the mean and the standard deviation. The 3rd-5th columns record the 0.25-quantile, 0.5-quantile and 0.75-quantile of the RE over 100 trials. In addition to the RE over 100 trials, we also record the maximal (in the 6th column LRE) and the minimal RE (in the 7th column SRE) among the 22 stocks. We find that CDAN achieves the best performance. Moreover, its Q2 RE is close to its mean RE, and the standard deviation of RE is quite small, showing that the CDAN model is pretty stable in terms of the model performance. We plot additional visualizing pictures in Figure \ref{Fig_rmse_md_cd} to have a better understanding of the results. In Figure \ref{Fig_rmse_md_cd}(a), CDAN converges fast in terms of the test RMSE, and it's efficient in controlling the relative errors. Diving deeper into the training details, from Fig \ref{Fig_rmse_md_cd}(b) we observe that the CD of CDAN decreases more significantly (to around $10^{-2}$) than that of DAN (to around $10^{-1}$). It shows that CDAN does capture the dependence difference which explains and contributes to the outperformance of CDAN.

\subsection{Wine quality regression}

The UCI wine quality dataset \cite{cortez2009modeling} contains records of red and white vinho verde wine samples from the north of Portugal, with sample size 1599 and 4598 respectively. Each record has 12 features, such as pH, alcohol and quality. The red wine and white wine samples differ in the feature distributions. Our goal is to predict the wine quality with two transfer tasks, from white wine to red wine (W$\rightarrow$R) and from red wine to white wine (R$\rightarrow$W).

We compare CDAN to 6 neural network baselines, namely MLP, AFN \cite{afn}, MCD \cite{mcd}, DANN \cite{ganin2016domain}, CORAL \cite{sun2016deep}, and DAN \cite{long2015learning}. Each neural network model has 2 hidden layers and each hidden layer has 8 units. For each model, we run 100 trials and record the RMSE, R2 scores, and relative errors. 
\begin{table}[!htbp]
\centering
\footnotesize
\caption{The RMSE, R2 score (R2) and relative error (RE) for wine quality prediction.}
\begin{tabular}{ccccccc} 
\toprule
& \multicolumn{3}{c}{W$\rightarrow$R} & \multicolumn{3}{c}{R$\rightarrow$W} \\
\cmidrule(lr){2-4} \cmidrule(lr){5-7}
            & RMSE & R2 & RE & RMSE & R2 & RE \\
\midrule
MLP & 0.125 & 0.131 & 0.110 & 0.143 &  0.067 & 0.115 \\
AFN & 0.129 & 0.087 & 0.119 & 0.145 &  0.032 & 0.118 \\
MCD & 0.125 & 0.137 & 0.109 & 0.144 &  0.042 & 0.116 \\
DANN & 0.127 & 0.104 & 0.115 & 0.147 &  0.006 & 0.119 \\
DAN & 0.122 & 0.175 & 0.109 & 0.138 &  0.123 & 0.112 \\
CORAL & 0.125 & 0.136 & 0.109 & 0.144 &  0.054 & 0.115 \\
CDAN & \textbf{0.120} & \textbf{0.201} & \textbf{0.108} & \textbf{0.133} &  \textbf{0.177} & \textbf{0.109} \\
\bottomrule
\end{tabular}
\label{Tb_comp3}
\end{table}

\begin{figure}[!htbp]
\centering
 \includegraphics[width=0.45\columnwidth]{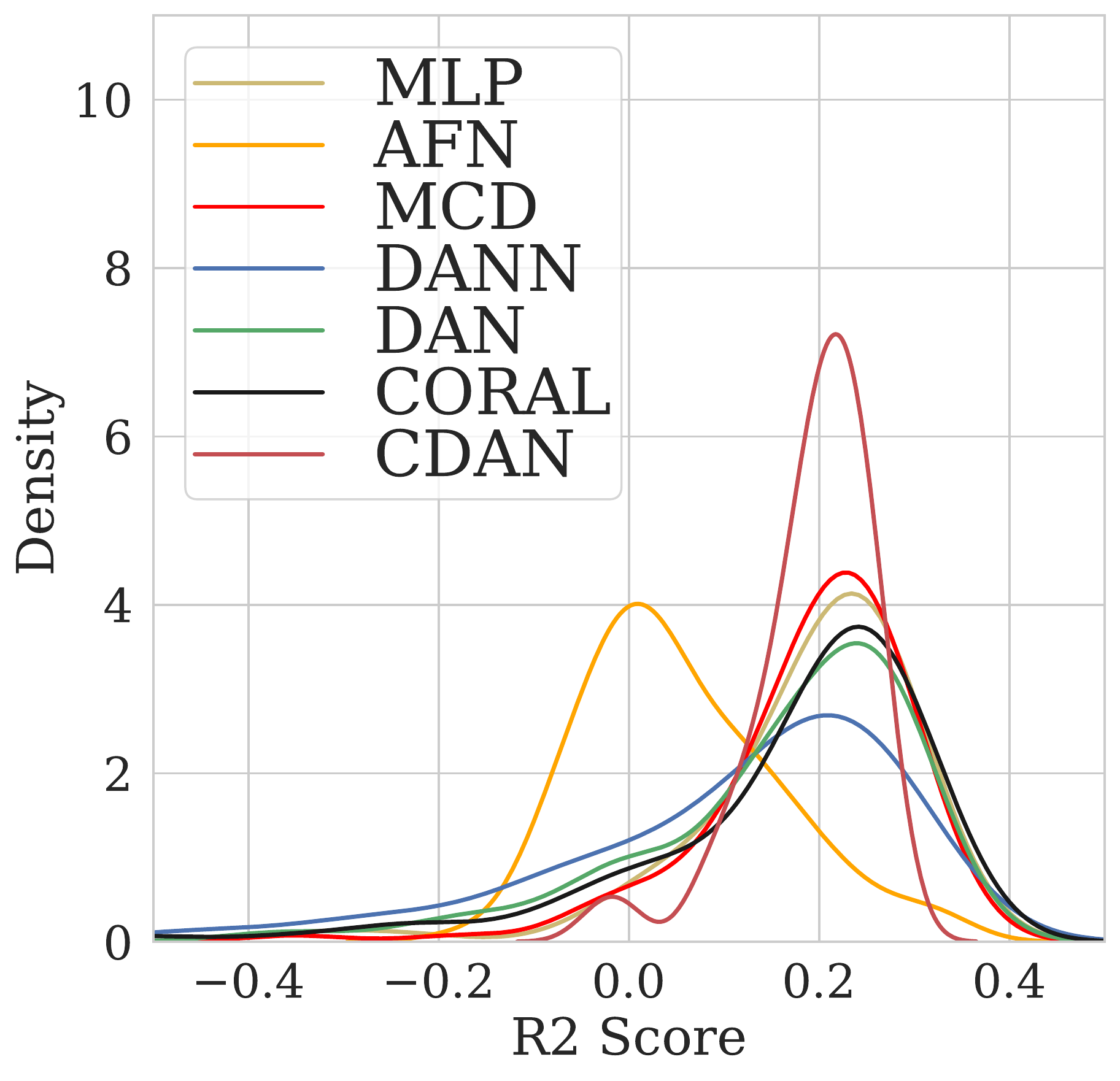}
 \includegraphics[width=0.45\columnwidth]{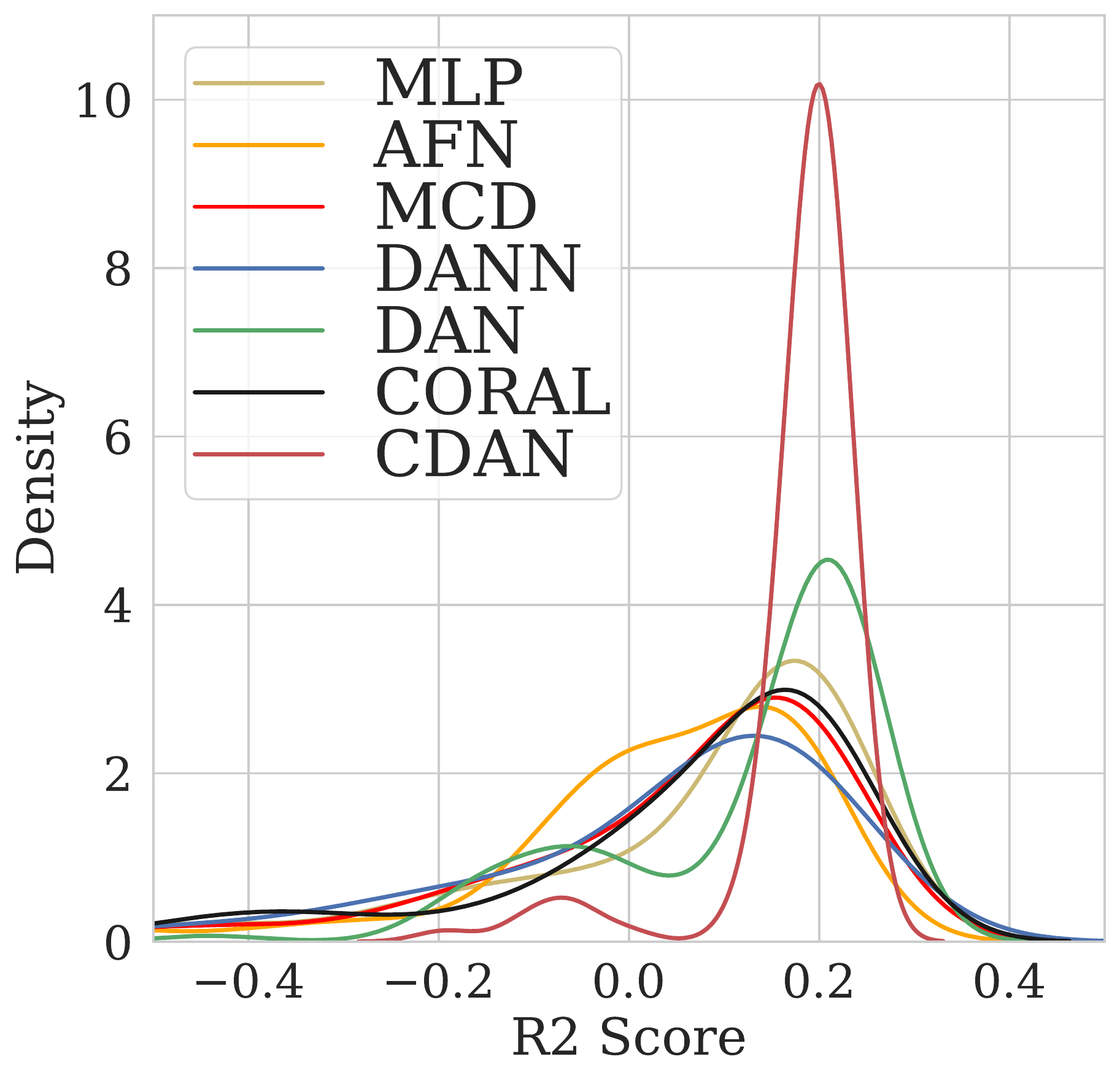}
 \caption{R2 score distributions over 100 runs for the transfer task W$\rightarrow$R (left) and R$\rightarrow$W (right).}
 \label{Fig_wine_re}
\end{figure}
\begin{table*}[htpb] 
\caption{Parameter sensitivity for the CDAN model in 8 transfer tasks on the retail credit dataset.}
\label{Tb_retail}
\centering
\footnotesize
\begin{tabular}{lllllllll}
\toprule
$(\alpha,\beta)$    &     19May $\to$ Jun         &    19Oct $\to$ Nov          &     19Dec $\to$ Jan         &     20Jan $\to$ Feb        & 20Feb $\to$ Mar             &     20Mar $\to$ Apr         &     20Apr $\to$ May          &      20May $\to$ Jun        \\
\midrule
(0.01,0.01) & 80.02$\pm$0.04 & 79.33$\pm$0.06 & \textbf{92.23$\pm$0.005} & 88.04$\pm$0.02 & 91.14$\pm$0.01 & 92.28$\pm$0.01 & 92.75$\pm$0.003 & 91.12$\pm$0.01 \\
(0.01,0.1) & 79.97$\pm$0.04 & 80.95$\pm$0.04 & 91.89$\pm$0.005 & 87.14$\pm$0.02  & 90.71$\pm$0.01  & 91.65$\pm$0.01  & 92.67$\pm$0.004 & 90.71$\pm$0.01  \\
(0.01, 1)  & 79.84$\pm$0.03 & \textbf{81.02$\pm$0.03} & 91.84$\pm$0.005 & 86.82$\pm$0.02 & 90.36$\pm$0.02 & 91.63$\pm$0.01  & 92.54$\pm$0.003 & 91.15$\pm$0.01  \\
(0.1,0.01) & 78.98$\pm$0.04 & 78.90$\pm$0.06 & 92.18$\pm$0.004 & 87.80$\pm$0.02 & 91.40$\pm$0.01  & 92.44$\pm$0.01 & \textbf{92.78$\pm$0.003} & 91.20$\pm$0.01  \\
(0.1,0.1)  & 79.79$\pm$0.03 & 79.95$\pm$0.05 & 91.85$\pm$0.006 & 86.90$\pm$0.02 & 90.94$\pm$0.01 & 91.85$\pm$0.01  & 92.68$\pm$0.003 & 90.91$\pm$0.01  \\
(0.1,1)    & 79.95$\pm$0.03 & 80.87$\pm$0.04 & 91.86$\pm$0.005 & 87.17$\pm$0.02  & 90.54$\pm$0.01 & 91.50$\pm$0.02  & 92.52$\pm$0.004 & 91.02$\pm$0.01  \\
(1,0.01)   & \textbf{80.44$\pm$0.04} & 80.79$\pm$0.06 & 91.93$\pm$0.005 & \textbf{88.20$\pm$0.02} & \textbf{91.51$\pm$0.01} & \textbf{92.46$\pm$0.01} & 92.74$\pm$0.003 & \textbf{91.39$\pm$0.01} \\
(1,0.1)    & 80.26$\pm$0.04 & 79.66$\pm$0.06 & 91.90$\pm$0.006 & 87.64$\pm$0.02 & 90.91$\pm$0.01  & 91.97$\pm$0.01  & 92.76$\pm$0.003 & 91.07$\pm$0.01 \\
\bottomrule
\end{tabular}
\end{table*}

\begin{table*}[htbp]
\caption{Parameter sensitivity for the CDAN model on the equity price dataset.}
\label{Tb_hyperparameter}
\centering
\footnotesize
\begin{tabular}{lccccccc}
\toprule
        $(\alpha,\beta)$ & RMSE     & RE(\%) & Q1 RE(\%) & Q2 RE(\%) & Q3 RE(\%) & LRE(\%) & SRE(\%)  \\
\midrule
(0.01,0.01) & \textbf{0.0235 $\pm$ 2.3e-3} & \textbf{1.77 $\pm$ 0.21}& \textbf{1.63} & \textbf{1.75} & \textbf{1.89} & \textbf{6.06 $\pm$ 1.53}& 0.43 $\pm$ 0.09 \\
(0.01, 0.1) & 0.0237 $\pm$ 1.9e-3 & 1.80 $\pm$ 0.19 & 1.66 & 1.79 & 1.92 & 6.14 $\pm$ 1.58 & \textbf{0.43 $\pm$ 0.08} \\
(0.1,0.01) & 0.0243 $\pm$ 2.6e-3 & 1.83 $\pm$ 0.26 & 1.66 & 1.79 & 1.98 & 6.52 $\pm$ 2.00 & 0.44 $\pm$ 0.10 \\
(0.1, 0.1) & 0.0243 $\pm$ 2.6e-3 & 1.84 $\pm$ 0.25 & 1.66 & 1.80 & 2.00 & 6.14 $\pm$ 1.51 & 0.44 $\pm$ 0.09 \\
(0.1, 1) & 0.0249 $\pm$ 2.5e-3 & 1.87 $\pm$ 0.26 & 1.70 & 1.86 & 2.03 & 6.40 $\pm$ 1.80 & 0.46 $\pm$ 0.09 \\
(1, 0.1) & 0.0247 $\pm$ 2.3e-3 & 1.87 $\pm$ 0.23 & 1.72 & 1.87 & 1.98 & 6.64 $\pm$ 1.86 & 0.44 $\pm$ 0.10 \\
(1, 1) & 0.0259 $\pm$ 2.4e-3 & 1.94 $\pm$ 0.23 & 1.80 & 1.95 & 2.06 & 6.68 $\pm$ 1.67 & 0.48 $\pm$ 0.09 \\
(1, 10) & 0.0306 $\pm$ 4.7e-3 & 2.18 $\pm$ 0.35  & 1.94 & 2.09 & 2.31 & 6.88 $\pm$ 2.01   & 0.59 $\pm$ 0.13   \\
(10, 1) & 0.0276 $\pm$ 3.2e-3 & 2.04 $\pm$ 0.27 & 1.83 & 2.01 & 2.20 & 6.87 $\pm$ 1.85 & 0.50 $\pm$ 0.11 \\
\bottomrule
\end{tabular}
\end{table*}

The results are summarized in Table \ref{Tb_comp3}. We conclude that the CDAN model outperforms the other benchmarks by achieving the highest R2 score, the smallest relative error and the smallest RMSE. Furthermore, we plot the R2 score distribution over the 100 runs for the two transfer tasks in Figure \ref{Fig_wine_re}. We see that the R2 score for CDAN model is more concentrated than its competitors, showing that the outperformance is quite stable.

\subsection{Parameter sensitivity and ablation study}

To further look into the sensitivity of the parameters $\alpha$ and $\beta$, we compare the model performance on the retail credit dataset and the equity price dataset. For the retail dataset, we list the model performance under various combinations of $\alpha$ and $\beta$ in the 8 transfer tasks in Table \ref{Tb_retail}. Among the best performances in each transfer task, we find that the coefficient of MD is in most cases larger than that of CD, indicating that the marginal differences and the dependence difference weigh differently in measuring the overall domain divergence. For the equity price dataset, we test 9 candidate pairs of hyperparameters $(\alpha, \beta)$ and record corresponding model performances in Table \ref{Tb_hyperparameter}. It shows that as the coefficients increase, the model performance tends to get worse. And the change in the CD parameter $\beta$ can significantly affect the model performance. It thus confirms the motivation of learning deep features by jointly adapting marginal divergence and copula distance, since a good trade-off between them could enhance feature transferability.

To evaluate the efficiency of taking MD and CD separately into the regularizer, we run experiments on the wine quality dataset. We compare the results of either $\alpha=0$ or $\beta=0$ and summarize them in Table \ref{Tb_comp4}. The ablation study shows that MD and CD are both essential in terms of a good model performance.

\begin{table}[!htbp]
\centering
\small
\caption{Ablation study on wine quality prediction.}
\begin{tabular}{ccc} 
\toprule
$(\alpha,\beta)$   & W$\rightarrow$R R2     &  R$\rightarrow$W R2     \\ 
\midrule
$(0,0)$ & 0.131 & 0.067 \\
$(0,0.1)$ & 0.145 & 0.012 \\
$(0,1)$ & 0.130 & 0.069 \\
$(0,10)$ & 0.121 & 0.070 \\
$(0.1,0)$ & 0.151 & 0.137 \\
$(1,0)$ & 0.181 & 0.158 \\
$(10,0)$ & 0.093 & 0.153\\
$(1,1)$ & \textbf{0.201} & \textbf{0.177} \\
\bottomrule
\end{tabular}
\label{Tb_comp4}
\end{table}

\begin{table}[!htbp]
\centering
\footnotesize
\caption{Comparison of different divergence measures for wine quality prediction.}
\begin{tabular}{lcccccc} 
\toprule
& \multicolumn{3}{c}{W$\rightarrow$R} & \multicolumn{3}{c}{R$\rightarrow$W} \\
\cmidrule(lr){2-4} \cmidrule(lr){5-7}
  $\mathcal{H}_1$ + $\mathcal{H}_2$          & RMSE & R2 & RE & RMSE & R2 & RE \\
\midrule
MMD + KL & \textbf{0.120} & 0.201 & \textbf{0.108} & \textbf{0.133} &  0.177 & \textbf{0.109} \\
MMD + W1 & 0.135 & 0.140 & 0.124 & 0.146 &  0.178 & 0.115 \\
MMD + $\chi^2$ & 0.135 & 0.143 & 0.124 & 0.147 &  0.176 & 0.114 \\
KL + KL & 0.128 & 0.175 & 0.112 & 0.150 &  0.036 & 0.118 \\
KL + W1 & 0.126 & \textbf{0.210} & 0.112 & 0.149 &  0.150 & 0.117 \\
KL + $\chi^2$ & 0.133 & 0.140 & 0.115 & 0.147 &  0.122 & 0.118 \\
W1 + KL & 0.124 & 0.191 & 0.110 & 0.145 &  0.112 & 0.115 \\
W1 + W1 & 0.126 & 0.208 & 0.113 & 0.144 & \textbf{0.179} & 0.114\\
W1 + $\chi^2$ & 0.125 & 0.200 & 0.110 & 0.144 & 0.151 & 0.114\\
\bottomrule
\end{tabular}
\label{Tb_divg}
\end{table}

\subsection{Comparison of different divergence measures}

As we have mentioned, there can be multiple choices over the divergence measures $\mathcal{H}_1$ and $\mathcal{H}_2$. In this section, we investigate the performance difference caused by the various divergence measures. Specifically, the candidate divergence measures for $\mathcal{H}_1$ include KL divergence, W1 (abbreviated for Wasserstein-1) distance and MMD. And the candidate divergence measures for $\mathcal{H}_2$ include KL divergence, $\chi^2$ (abbreviated for Pearson $\chi^2$) divergence and W1 distance. In Table \ref{Tb_divg}, we record the model performance of CDAN with the various combinations of $\mathcal{H}_1$ and $\mathcal{H}_2$ on the UCI wine quality dataset. From the table, we see that, the CDAN model with $\mathcal{H}_1$ taking MMD distance and $\mathcal{H}_2$ taking KL divergence performs the best in terms of RMSE and RE. Also, it should be noted that the performance difference caused by the divergence measures can be as large as that brought by different models. That reminds us to be prudent in choosing the suitable divergence measures.

\section{Conclusion}

This work proposes a new domain adaptation framework that facilitates a user to detect whether the domain difference in a transfer task comes from the marginals' differences or the dependence difference. Specifically, we quantify the dependence difference with copula distance, a difference measure endowed with boundedness and monotonicity to guarantee the algorithm convergence. By optimizing the relative weights between the marginal divergence and the copula distance, we can acquire transferability across domains in a more flexible way. Experiments on the real-world datasets demonstrate the efficacy and robustness of our approach compared to a variety of existing domain adaptation models.

\section*{Acknowledgments}

Shumin Ma acknowledges the support from: Guangdong Provincial Key Laboratory of Interdisciplinary Research and Application for Data Science, BNU-HKBU United International College (2022B1212010006), Guangdong Higher Education Upgrading Plan (2021-2025) (UIC R0400001-22) and UIC (UICR0700019-22). Qi Wu acknowledges the support from the Hong Kong Research Grants Council [General Research Fund 14206117, 11219420, and 11200219], CityU SRG-Fd fund 7005300, and the support from the CityU-JD Digits Laboratory in Financial Technology and Engineering, HK Institute of Data Science. The work described in this paper was partially supported by the InnoHK initiative, The Government of the HKSAR, and the Laboratory for AI-Powered Financial Technologies. 

{\appendices

\section*{Proof of Proposition \ref{boundedness} and \ref{monotonicity}}
\begin{proof} We begin with the analysis over the explicit form of the bivariate copula distance $CD_{\mathcal{H}}(\mathbf{X,Y})$ when $\mathcal{H}$ is taken to be different divergence measures. Note that the copula distance between multivariate distributions is defined in terms of bi-variate sub-distributions. Thus, it is enough to prove the boundedness and monotonicity of $CD_{\mathcal{H}}(\mathbf{X,Y})$ between any two bivariate random vectors $\mathbf{X}$, $\mathbf{Y} \in \mathbb{R}^2$. Suppose that the Gaussian copula parameters for $\mathbf{X}$ and $\mathbf{Y}$ are $\Sigma^\mathbf{X}$ and $\Sigma^\mathbf{Y}$, respectively. Their copula density functions are:
\begin{equation}\label{EqGauss}
\begin{aligned}
&c^\mathbf{X}(u_1,u_2) = |\Sigma^{\mathbf{X}}|^{-\frac{1}{2}} \exp\Big(-\dfrac{1}{2}\mathbf{x}^T\big((\Sigma^{\mathbf{X}})^{-1}-I\big)\mathbf{x}\Big),\\
&c^\mathbf{Y}(u_1,u_2) = |\Sigma^{\mathbf{Y}}|^{-\frac{1}{2}} \exp\Big(-\dfrac{1}{2}\mathbf{x}^T\big((\Sigma^{\mathbf{Y}})^{-1}-I\big)\mathbf{x}\Big),
\end{aligned}
\end{equation}
where $\mathbf{x} := [x_1, x_2]^T=[\Phi^{-1}(u_1),\Phi^{-1}(u_2)]^T$ with $\Phi$ being the CDF of the standard normal distribution. 

The first divergence class is $\phi$-divergence (see \cite{7552457} for the detailed descriptions of the $\phi$-divergence family). Given a convex function $\phi(x)$ such that $\phi(1)=0$, the $\phi$ divergence between two distributions $P^\mathbf{X}$ and $P^\mathbf{Y}$ is defined by $\mathcal{H}_\phi(P^\mathbf{X}, P^\mathbf{Y}) = \int \phi(\frac{dP^\mathbf{X}}{dP^\mathbf{Y}}) dP^\mathbf{Y}$. With the following proposition, we prove that the copula distance between bivariate random vectors $\mathbf{X}$ and $\mathbf{Y}$, $CD_{\mathcal{H}_{\phi}}(\mathbf{X,Y})$, can be fully characterized by the copula density functions $c^{\mathbf{X}}$ and $c^{\mathbf{Y}}$.
\begin{proposition} \label{prop_cd}
For any bivariate random vector $\mathbf{X}\in \mathbb{R}^2$, the $\phi$-divergence between the probability distribution $P^{\mathbf{X}}$ and the product of marginal distributions $P^\mathbf{X}_1 P^\mathbf{X}_2$ is, 
\begin{equation*}
\small
    \mathcal{H}_\phi(P^\mathbf{X},P^\mathbf{X}_1 P^\mathbf{X}_2) = \int_0^1 \int_0^1 \phi \big(c^\mathbf{X}(u_1,u_2)\big) du_1 du_2.
\end{equation*}
For any two bivariate random vectors $\mathbf{X}$, $\mathbf{Y} \in \mathbb{R}^2$, the copula distance between $\mathbf{X}$ and $\mathbf{Y}$ when $\mathcal{H}$ takes $\phi$-divergence is
\begin{equation*}
\begin{aligned}
\small
& CD_{\mathcal{H}_\phi}(\mathbf{X,Y}) \\
= & |\int_0^1 \int_0^1 \phi \big(c^\mathbf{X}(u_1,u_2)\big) - \phi\big(c^\mathbf{Y}(u_1,u_2)\big) du_1du_2|.
\end{aligned}
\end{equation*}
\end{proposition}
\begin{proof}
We denote the two marginal density functions for the bivariate random vector $\mathbf{X}$ as $p_1^\mathbf{X}(\cdot)$ and $p_2^\mathbf{X}(\cdot)$. By the definition of copula density function, the $\phi$-divergence between $P^\mathbf{X}$ and $P^\mathbf{X}_1 P^\mathbf{X}_2$ is, 
\begin{equation*}
\small
\begin{aligned}
& \mathcal{H}_\phi(P^\mathbf{X},P^\mathbf{X}_1 P^\mathbf{X}_2) \\
= & \int \phi\big(\frac{p_1^\mathbf{X}(x_1)p_2^\mathbf{X}(x_2)c^\mathbf{X}(P_1^\mathbf{X}(x_1),P_2^\mathbf{X}(x_2))}{p_1^\mathbf{X}(x_1)p_2^\mathbf{X}(x_2)}\big) dP_1^\mathbf{X}(x_1)dP_2^\mathbf{X}(x_2) \\
= & \int \phi\big(c^\mathbf{X}(P_1^\mathbf{X}(x_1),P_2^\mathbf{X}(x_2))\big) dP_1^\mathbf{X}(x_1)dP_2^\mathbf{X}(x_2).
\end{aligned}
\end{equation*}
With change of variables $u_1 = P_1^{\mathbf{X}}(x_1)$ and $u_2 = P_2^{\mathbf{X}}(x_2)$, we finally have $\mathcal{H}_\phi(P^\mathbf{X},P^\mathbf{X}_1 P^\mathbf{X}_2) = \int_0^1 \int_0^1 \phi(c^\mathbf{X}(u_1,u_2)) du_1 du_2$. For the bivariate random vector $\mathbf{Y}$, we similarly have $\mathcal{H}_\phi(P^\mathbf{Y},P^\mathbf{Y}_1 P^\mathbf{Y}_2) = \int_0^1 \int_0^1 \phi(c^\mathbf{Y}(u_1,u_2)) du_1 du_2$. The copula distance between $\mathbf{X}$ and $\mathbf{Y}$ is defined as the absolute difference between $\mathcal{H}_\phi(P^\mathbf{X},P^\mathbf{X}_1 P^\mathbf{X}_2)$ and $\mathcal{H}_\phi(P^\mathbf{Y},P^\mathbf{Y}_1 P^\mathbf{Y}_2)$. That completes the proof.
\end{proof}

With Proposition \ref{prop_cd}, we can directly obtain the results in Section 3 in the main paper:
\begin{itemize}
\setlength{\itemsep}{0pt}
\setlength{\parsep}{0pt}
\setlength{\parskip}{0pt}
\item[-] When $\phi(x)=x^2-1$, the resulting $\phi$-divergence is a $\chi^2$ distance. Thus, $\mathcal{H}_{\chi^2}(P_{ij},P_i P_j) = \int_0^1 \int_0^1 (c_{ij}^2(u_i,u_j)-1)\mathrm{d}u_i \mathrm{d}u_j$.
\item[-] When $\phi(x)=(\sqrt{x}-1)^2$, it corresponds to Hellinger distance. So we have $\mathcal{H}_{H}(P_{ij},P_i P_j) = \int_0^1 \int_0^1 [\sqrt{c_{ij}(u_i,u_j)}-1]^2\mathrm{d}u_i \mathrm{d}u_j$.
\item[-] When $\phi(x)=\frac{x(1-x^{-(\alpha+1)/2})}{1-\alpha^2}$, it results in $\alpha$-divergence. Thus, $\mathcal{H}_{\alpha}(P_{ij},P_i P_j) = \frac{1}{1-\alpha^2}\int_0^1 \int_0^1 [1-{c_{ij}(u_i,u_j)}^{-\frac{\alpha+1}{2}}]c_{ij}(u_i,$ $u_j)\mathrm{d}u_i \mathrm{d}u_j$.
\end{itemize}

Proposition \ref{prop_cd} states that the copula distance defined by the $\phi$-divergence is a function of the copula densities. Also, it proves that the $\phi$-divergence between the joint distribution and the product of marginals is purely a function of the copula densities. That is to say, when the divergence metric is taken to be $\phi$-divergence, the calculation of the copula distance has nothing to do with the marginal distributions. Thus, in the following proofs, when calculating the copula distance between any two random vectors $\mathbf{X}$ and $\mathbf{Y}$, we will assume the marginals are both standard normal distributions that has mean 1 and variance 0. That will greatly simplify our calculation. We can just calculate the copula distance between two bivariate Gaussian vectors with copula densities $c^{\mathbf{X}}(\mathbf{u})$ and $c^{\mathbf{Y}}(\mathbf{u})$, respectively. Furthermore, given that $\mathbf{X}$ is Gaussian with standard normal marginals, we know that their Gaussian copula parameter $\Sigma^{\mathbf{X}}$ is exactly the correlation matrix (\cite{dalla2009bayesian}). In the following proof, we will write $\Sigma^\mathbf{X} := \begin{gathered}
\begin{pmatrix} 1 & \rho \\ \rho & 1 \end{pmatrix}
\end{gathered}$, with $\rho \in [-1,1]$. Now we are ready to provide the explicit forms of the copula distance for various choices of the divergence measures.
 
\textbf{KL divergence.} When $\phi(x)=x\log x$, the corresponding $\phi$-divergence is KL divergence. By definition, we have
\begin{equation*}
\footnotesize
\mathcal{H}_{\textup{KL}}(P^\mathbf{X},P^\mathbf{X}_1 P^\mathbf{X}_2)
= \iint p^\mathbf{X}(x_1,x_2) \log \frac{p^\mathbf{X}(x_1,x_2)}{p^\mathbf{X}_1(x_1)p^\mathbf{X}_2(x_2)} dx_1dx_2.
\end{equation*}
Given that $p^\mathbf{X}(x_1,x_2) = p^\mathbf{X}(x_1) p^\mathbf{X}(x_2) c^\mathbf{X}(u_1,u_2)$, with Eq. \eqref{EqGauss}, we have
\begin{equation*}
\footnotesize
\begin{aligned}
& \mathcal{H}_{\textup{KL}}(P^\mathbf{X},P^\mathbf{X}_1 P^\mathbf{X}_2) \\
=& \iint \frac{p^{\mathbf{X}}(x_1,x_2)([x_1,x_2](\mathbf{I}-(\Sigma^\mathbf{X})^{-1})[x_1,x_2]^T-\log|\Sigma^\mathbf{X}|)}{2} dx_1 dx_2 \\
=& \frac{1}{2} \mathbb{E}_{p^{\mathbf{X}}}([x_1,x_2](\mathbf{I}-(\Sigma^\mathbf{X})^{-1})[x_1,x_2]^T-\log|\Sigma^\mathbf{X}|)  \\
=& \frac{1}{2} (-\log |\Sigma^\mathbf{X}|+2-2) \\
=& -\frac{1}{2} \log |\Sigma^\mathbf{X}|.
\end{aligned}
\end{equation*}
The third equality comes from \cite{IMM2012-03274}, where it proves that $\mathbb{E}_{p^{\mathbf{X}}}([x_1,x_2](\Sigma^\mathbf{X})^{-1}[x_1,x_2]^T) = 2.$ Finally, with the definition of the copula distance, we have:
$$CD_{\mathcal{H}_{\textup{KL}}}(\mathbf{X,Y}) = \frac{1}{2}|\log|\Sigma^\mathbf{X}| - \log|\Sigma^\mathbf{Y}||.$$

\textbf{$\chi^2$ distance.} When $\phi(x)=x^2-1$, the resulting $\phi$-divergence is a $\chi^2$ distance. By definition, we have
\begin{equation*}
\begin{aligned}
&\mathcal{H}_{\chi^2}(P^\mathbf{X},P^\mathbf{X}_1 P^\mathbf{X}_2) \\
= & \iint (\frac{p^{\mathbf{X}}(x_1,x_2)}{p^{\mathbf{X}}_1(x_1)p^{\mathbf{X}}_2(x_2)})^2 p^{\mathbf{X}}_1(x_1)p^{\mathbf{X}}_2(x_2) dx_1 dx_2 - 1.
\end{aligned}
\end{equation*}
Since $\frac{p^{\mathbf{X}}(x_1,x_2)}{p^{\mathbf{X}}_1(x_1)p^{\mathbf{X}}_2(x_2)} = c^\mathbf{X}(u_1,u_2)$ and $p^{\mathbf{X}}_1(x) = p^{\mathbf{X}}_2(x) = \frac{1}{\sqrt{2\pi}}\exp{(-\frac{x^2}{2})}$, we can further simplify the calculation as:
\begin{equation*}
\begin{aligned}
& \mathcal{H}_{\chi^2}(P^\mathbf{X},P^\mathbf{X}_1 P^\mathbf{X}_2) \\
= & \iint \frac{\exp\Big([x_1,x_2] \big(\frac{I}{2}-(\Sigma^\mathbf{X})^{-1}\big)[x_1,x_2]^T\Big)}{2\pi |\Sigma^\mathbf{X}|}  dx_1dx_2  - 1 \\
= & |\Sigma^\mathbf{X}|^{-1} - 1.
\end{aligned}
\end{equation*}
The last equality comes from the following fact: 
$2(\Sigma^\mathbf{X})^{-1}-I $ is positive definite with determinant $1$. That gives:
$\iint \exp \big([x_1,x_2]\big(\frac{I}{2}-(\Sigma^\mathbf{X})^{-1}\big)[x_1,x_2]^T \big) dx_1dx_2 = 2\pi.$ Finally, we have:$$CD_{\mathcal{H}_{\chi^2}}(\mathbf{X,Y}) = ||\Sigma^\mathbf{X}|^{-1} - |\Sigma^\mathbf{Y}|^{-1}|.$$

It is not hard to derive more results about the copula distance for other $\phi$-divergence. So we omit them here and turn to the derivation for Wasserstein-2 distance and MMD distance.

\textbf{Wasserstein-2 distance.} Assume that the marginal distributions of $\mathbf{X,Y}$ are standard normals.
By directly applying the conclusion in Proposition 7 in \cite{wasserstein2}, we have
\begin{equation*}
\mathcal{H}_{\textup{W}}^2(P^\mathbf{X},P^\mathbf{X}_1 P^\mathbf{X}_2) = 4 - 2\textup{Tr}\big((\Sigma^\mathbf{X})^{\frac{1}{2}}\big) = 4 - 2\sqrt{2+2\sqrt{|\Sigma^\mathbf{X}|}}.
\end{equation*}

Thus,
\begin{equation*}
\begin{aligned}
& CD_{\mathcal{H_{\textup{W}}}}(\mathbf{X,Y})\\ = & \Big|\sqrt{4 - 2\sqrt{2+2\sqrt{|\Sigma^\mathbf{X}|}}} - \sqrt{4 - 2\sqrt{2+2\sqrt{|\Sigma^\mathbf{Y}|}}}\Big|.
\end{aligned}
\end{equation*}

\textbf{Gaussian MMD distance.} Assume that the marginal distributions of $\mathbf{X,Y}$ are standard normals.
For ease of calculation, we take the simplest kernel function $k(\mathbf{X},\mathbf{Y})=e^{-||\mathbf{X}-\mathbf{Y}||_2^2}$. But we emphasize that, the following calculation applies to all Gaussian kernels. Using the kernel trick, the squared MMD distance can be computed as the expectation of kernel functions:
\begin{equation}\label{Eq_MMD}
\begin{aligned}
& \mathcal{H}_{\textup{MMD}}^2(P^\mathbf{X},P^\mathbf{X}_1 P^\mathbf{X}_2) \\
= & \mathbb{E}_{\mathbf{X},\mathbf{X}}k(\mathbf{X},\mathbf{X}) + \mathbb{E}_{\dot{\mathbf{X}},\dot{\mathbf{X}}}k(\dot{\mathbf{X}},\dot{\mathbf{X}}) -2\mathbb{E}_{\mathbf{X},\dot{\mathbf{X}}}k(\mathbf{X},\dot{\mathbf{X}}), 
\end{aligned}
\end{equation}
where $\dot{\mathbf{X}} \in \mathbb{R}^2$ is a random Gaussian vector with CDF $P^{\dot{\mathbf{X}}}(x_1,x_2)=P_1^{\mathbf{X}}(x_1)P_2^{\mathbf{X}}(x_2)$ and the Gaussian copula parameter $\Sigma^{\dot{\mathbf{X}}} = \begin{gathered} \begin{pmatrix} 1 & 0 \\ 0 & 1 \end{pmatrix}\end{gathered}.$

From Eq. \eqref{Eq_MMD}, we know that to calculate the squared MMD distance $\mathcal{H}_{\textup{MMD}}^2(P^\mathbf{X},P^\mathbf{X}_1 P^\mathbf{X}_2)$, we need to calculate the three expectations on the right-hand-side of this equation. We begin with the calculation of $\mathbb{E}_{\mathbf{X},\dot{\mathbf{X}}}k(\mathbf{X},\dot{\mathbf{X}})$. By definition, we have:
\begin{equation*}
\footnotesize
\begin{aligned}
& \mathbb{E}_{\mathbf{X},\dot{\mathbf{X}}}k(\mathbf{X},\dot{\mathbf{X}})\\ 
=& \iint \frac{k(\mathbf{X},\mathbf{Y}) \exp\big(-\frac{1}{2}\mathbf{X}^T(\Sigma^{\mathbf{X}})^{-1}\mathbf{X}-\frac{1}{2}\mathbf{Y}^T(\Sigma^{\dot{\mathbf{X}}})^{-1}\mathbf{Y}\big)}{4\pi^2\sqrt{|\Sigma^{\mathbf{X}}\Sigma^{\dot{\mathbf{X}}}|}}  d\mathbf{X}d\mathbf{Y} \\
=&  \iiiint \frac{\exp(-\dfrac{1}{2}[x_1,x_2,y_1,y_2]A[x_1,x_2,y_1,y_2]^T)}{4\pi^2 \sqrt{|\Sigma^\mathbf{X} \Sigma^{\dot{\mathbf{X}}}|}}
 dx_1dx_2dy_1dy_2 \\
=& \dfrac{1}{4\pi^2 \sqrt{|\Sigma^\mathbf{X} \Sigma^{\dot{\mathbf{X}}}|}} \times (2\pi)^2 |A|^{-\frac{1}{2}} \\
=& \frac{1}{\sqrt{|2\Sigma^\mathbf{X}+2\Sigma^{\dot{\mathbf{X}}}+I|}} = \frac{1}{\sqrt{21 + 4|\Sigma^\mathbf{X}|}}. 
\end{aligned}
\end{equation*}
Here, the matrix $A := \begin{gathered} \begin{pmatrix} (\Sigma^\mathbf{X})^{-1}+2I & -2I \\ -2I & (\Sigma^{\dot{\mathbf{X}}})^{-1} +2I \end{pmatrix}\end{gathered} \in \mathbb{R}^{4 \times 4}$ is positive semidefinite with determinant $\frac{|2\Sigma^\mathbf{X}+2\Sigma^{\dot{\mathbf{X}}}+I_2|}{|\Sigma^\mathbf{X}\Sigma^{\dot{\mathbf{X}}}|}$. Similarly, we have:
\begin{equation*}
\begin{aligned}
\mathbb{E}_{\mathbf{X},\mathbf{X}}k(\mathbf{X},\mathbf{X}) &= \frac{1}{\sqrt{|4\Sigma^\mathbf{X}+I|}} = \frac{1}{\sqrt{9+16|\Sigma^\mathbf{X}|}}, \\
\mathbb{E}_{\dot{\mathbf{X}},\dot{\mathbf{X}}}k(\dot{\mathbf{X}},\dot{\mathbf{X}}) &= \frac{1}{\sqrt{|4\Sigma^{\dot{\mathbf{X}}}+I|}} = \frac{1}{5}.
\end{aligned}
\end{equation*}

Organizing the three terms together, we have the squared MMD distance: 
\begin{equation*}
\mathcal{H}^2_{\textup{MMD}}(P^\mathbf{X},P^\mathbf{X}_1 P^\mathbf{X}_2) = \frac{1}{\sqrt{9+16|\Sigma^\mathbf{X}|}} + \frac{1}{5} - \frac{2}{\sqrt{21 + 4|\Sigma^\mathbf{X}|}},
\end{equation*}
and the copula distance
\begin{equation*}
\small
\begin{aligned}
    CD_{\mathcal{H}_{\textup{MMD}}}(\mathbf{X},\mathbf{Y}) &=  \Big|\sqrt{\frac{1}{\sqrt{9+16|\Sigma^\mathbf{X}|}}+ \frac{1}{5} -\frac{2}{\sqrt{21+4|\Sigma^\mathbf{X}|}}} \\
    & - \sqrt{\frac{1}{\sqrt{9+16|\Sigma^\mathbf{Y}|}}+ \frac{1}{5} -\frac{2}{\sqrt{21+4|\Sigma^\mathbf{Y}|}}}\Big|.
\end{aligned}
\end{equation*}

\textbf{Boundedness.} Given that $\Sigma^\mathbf{X}$ and $\Sigma^\mathbf{Y}$ for Gaussian random vectors are in essence the correlation matrix, we know that $|\Sigma^\mathbf{X}|\leq 1$ and $|\Sigma^\mathbf{Y}|\leq 1$. Thus, it is easy to verify that when $\mathcal{H}$ is Wasserstein-2 distance or Gaussian MMD distance, the copula distance is bounded. Furthermore, we know that the divergence measures (including the total variation distance, Hellinger distance, Jensen-Shannon divergence, etc.) are bounded by definition. Consequently, the corresponding copula distance is bounded.

\textbf{Monotonicity.} We fix the Gaussian copula parameter $\Sigma^\mathbf{Y}$ and express $CD_\mathcal{H}(\mathbf{X,Y})$ as a function of $\Sigma_{12}^\mathbf{X} = \rho$. A simple observation is that, if a function $f(x)$ is monotonically increasing with respect to $x$, then given $y$ fixed, $|f(x)-f(y)|$ is monotonically increasing with respect to $|x-y|$. So if $\mathcal{H}(P^\mathbf{X},P^\mathbf{X}_1 P^\mathbf{X}_2)$ is increasing with respect to $\rho^2$, we can conclude that the corresponding copula distance is monotonically increasing with $|(\Sigma_{12}^\mathbf{X})^2-(\Sigma_{12}^\mathbf{Y})^2|$. We check them one by one. 
\begin{itemize}
\setlength{\itemsep}{0pt}
\setlength{\parsep}{0pt}
\setlength{\parskip}{0pt}
\item[-] $\mathcal{H}_{\text{KL}}(P_{12},P_1 P_2) = -\frac{1}{2}(1-\rho^2)$ monotonically increases with $\rho^2$.
\item[-] $\mathcal{H}_{\chi^2}(P_{12},P_1 P_2) = \frac{1}{1-\rho^2}-1$ monotonically increases with $\rho^2$.
\item[-] $\mathcal{H}_{\textup{W}}^2(P^\mathbf{X},P^\mathbf{X}_1 P^\mathbf{X}_2) = 4-2\sqrt{2+2\sqrt{1-\rho^2}}$ monotonically increases with $\rho^2$.
\item[-] $\mathcal{H}^2_{\textup{MMD}}(P^\mathbf{X},P^\mathbf{X}_1 P^\mathbf{X}_2) = \frac{1}{\sqrt{25-16\rho^2}} + \frac{1}{5} - \frac{2}{\sqrt{25- 4\rho^2}}$.
\end{itemize}

From the above calculations, we conclude that $\mathcal{H}(P^\mathbf{X},P^\mathbf{X}_1 P^\mathbf{X}_2)$ is increasing with respect to $\rho^2$ when $\mathcal{H}$ is KL divergence, $\chi^2$ distance and Wasserstein-2 distance. For Gaussian MMD distance, we consider the function $f(x) = \frac{1}{\sqrt{25-16x}} - \frac{2}{\sqrt{25-4x}}, x \in [0,1]$. The first derivative $f'(x)=8(25-16x)^{-3/2} - 4(25-4x)^{-3/2} > 0$, suggesting that $\mathcal{H}_{\textup{MMD}}(P^\mathbf{X},P^\mathbf{X}_1 P^\mathbf{X}_2)$ is increasing with respect to $\rho^2$.
\end{proof}

\section*{Proof of Proposition \ref{prop_Gaussian}}
\begin{proof}

Without loss of generality, we assume that the probability density function $p(x_1,x_2)$ of $[X_1,X_2]$ is continuous. Consider the area $A_\delta=\{(x_1,x_2,\widetilde{x}_1,\widetilde{x}_2): |x_1-\widetilde{x}_1| \leq \delta \textup{ or } |x_2-\widetilde{x}_2| \leq \delta\} \subset \mathbb{R}^4$. Define $c_{\delta}:= \iiiint_{A_\delta} p(x_1,x_2)p(\widetilde{x}_1,\widetilde{x}_2)dx_1dx_2d\widetilde{x}_1d\widetilde{x}_2$. It always holds that, 
\begin{equation*}
\begin{aligned}
& \rho(X_1,X_2;a) = \mathbb{E}[\tanh\big(a(X_1 - \widetilde{X}_1)(X_2- \widetilde{X}_2)\big)]\\
=& (\iiiint_{A_\delta} + \iiiint_{\mathbb{R}^4-A_\delta}) \tanh\big(a(x_1-\widetilde{x}_1)(x_2-\widetilde{x}_2)\big) \\ 
&\times p(x_1,x_2)p(\widetilde{x}_1,\widetilde{x}_2) d\widetilde{x}_2d\widetilde{x}_1dx_2dx_1.
\end{aligned}
\end{equation*}

Outside $A_\delta$, we have $|a(x_1-\widetilde{x}_1)(x_2-\widetilde{x}_2)| \geq a\delta^2$, and 
\begin{equation*}
\footnotesize
\begin{aligned}
&\lim \limits_{a\to\infty} \iiiint_{\mathbb{R}^4-A_\delta} \tanh\big(a(x_1-\widetilde{x}_1)(x_2-\widetilde{x}_2)\big)\\
&\times p(x_1,x_2)p(\widetilde{x}_1,\widetilde{x}_2) dx_1dx_2d\widetilde{x}_1d\widetilde{x}_2 \\
=& \lim \limits_{a\to\infty} \iiiint_{\mathbb{R}^4-A_\delta} \textup{sign}\big(a(x_1-\widetilde{x}_1)(x_2-\widetilde{x}_2)\big) \\
&\times p(x_1,x_2)p(\widetilde{x}_1,\widetilde{x}_2) dx_1dx_2d\widetilde{x}_1d\widetilde{x}_2 \quad \text{(dominated convergence theorem)} \\
=& \rho_\tau(X_1,X_2) - \lim \limits_{a\to\infty}\iiiint_{A_\delta} \textup{sign}\big(a(x_1-\widetilde{x}_1)(x_2-\widetilde{x}_2)\big)\\
&\times p(x_1,x_2)p(\widetilde{x}_1,\widetilde{x}_2) dx_1dx_2d\widetilde{x}_1d\widetilde{x}_2.
\end{aligned}
\end{equation*}

Notice that inside $A_{\delta}$, $|\tanh(x)| \leq 1$. Thus,  $\forall \delta$,
\begin{equation*}
\footnotesize
\begin{aligned}
&|\lim \limits_{a\to\infty} \rho(X_1,X_2;a)-\rho_\tau(X_1,X_2)| \\
\leq & \lim \limits_{a\to\infty}\iiiint_{A_\delta} \Big( |\textup{sign}\big(a(x_1-\widetilde{x}_1)(x_2-\widetilde{x}_2)\big)| + \\ &|\tanh\big(a(x_1-\widetilde{x}_1)(x_2-\widetilde{x}_2)\big)|\Big)p(x_1,x_2)p(\widetilde{x}_1,\widetilde{x}_2) dx_1dx_2d\widetilde{x}_1d\widetilde{x}_2 \\
\leq & 2\iiiint_{A_\delta} p(x_1,x_2)p(\widetilde{x}_1,\widetilde{x}_2) dx_1dx_2d\widetilde{x}_1d\widetilde{x}_2 = 2 c_{\delta}.
\end{aligned}
\end{equation*}


Indeed, it is easy to verify that $c_0=0$, $c_\delta$ is finite and is continuous with respect to $\delta$.
Thus, we let $\delta$ approach zero and get the desired result that $\textup{lim}_{a \to \infty} \rho(X_1,X_2;a) = \rho_\tau(X_1,X_2).$
\end{proof}}

\section*{Implementation Details}

All experimental models were trained using the Adam optimizer implemented by Pytorch with initial learning rate 0.01. All activation functions were taken to be the ReLU. 

\subsection{Toy problem}

The MLP model is a simple neural network with 2 hidden layers of 8 and 4 neurons respectively.

In the DAN model, the feature extractor is a neural network with 2 hidden layers of 8 and 4 neurons respectively. The domain discrepancy is evaluated by the MMD distance and the discriminator is the final output layer.

The CORAL model has the same architecture with the DAN model, except for that the domain discrepancy is measured by the Frobenius norm.

The CDAN model differs with the CORAL model only in the measurement of the domain discrepancy.

\subsection{Retail credit classification}

The MLP model is a simple neural network with 2 hidden layers of 128 and 64 neurons respectively.

In the DAN model, the feature extractor is a neural network with 3 hidden layers of 64, 32 and 16 neurons respectively. The domain discrepancy is evaluated by the MMD distance of the 16-dimensional feature representations between the source and the target domains. The discriminator is the final output layer.

The AFN model contains 3 hidden layers with 128, 64 and 64 neurons respectively. The output tensors of the second hidden layer are aligned to a scale vector, which is pre-determined according to \cite{afn}.

In MCD model, the neural network consists of 3 hidden layers with 128, 64 and 64 neurons respectively. For output tensors of the second hidden layer from the target domain, two additional networks are used to construct the regularization term of \cite{mcd}.

The CORAL model is essentially the same model as the DAN model, except that the domain discrepancy is evaluated by the Frobenius norm of the covariance matrices of the 16-dimensional feature representations.

In the CDAN model, the feature extractor is a neural network with 6 hidden layers of 128, 128, 128, 128, 64 and 8 neurons respectively. The marginal divergence is calculated by the Gaussian MMD distance of each dimensional representations among the 8-dimensional features. The copula distance is with respect to the KL divergence. We run the model for 100 trials, and each trial costs about 2-3 minutes.

\subsection{Intra-day equity price regression}

After separating the historical prices of 22 stocks into two domains, we slice the data into pieces with length 12 and package them into batches of size 1024. For each stock in either domain, we use the MinMaxScaler to normalize its price. For each model, we use an LSTM layer as the feature extractor, and conduct the batch normalization for each Linear layer. We train each model for at most 100 epochs, and the early-stopping threshold is set to be 20 epochs. We tune the hyperparameters by grid search, and we also fine-tune the network parameters (including but not limited to number of layers, number of units, etc.).

The LSTM (RNN resp.) model consists of one LSTM (RNN resp.) layer of hidden size 64, and two Linear layers of size 64 and 32 respectively.

The DANN model consists of 3 neural networks, namely a feature extractor, a discriminator and a regressor. The feature extractor contains one LSTM layer of hidden size 64 and a Linear layer of size also 64. The discriminator is a binary classifier, which consists of three Linear layers of size 64, 32 and 16 respectively. The regressor consists of two Linear layers of size 64 and 32 respectively.

The CORAL model consists of one LSTM layer of hidden size 64, and 3 Linear layers of size 64, 32, 16 respectively. After extracting the features by LSTM, we calculate the regularization term according to \cite{sun2016return}.

The DAN model consist of one LSTM layer of hidden size 64, and 3 Linear layers of size 64, 32, 16 respectively. After extracting the features by LSTM, we calculate the MMD with a two-Gaussian-kernel function.

Our model CDAN consists of one LSTM layer of hidden size 64, and two Linear layers of size 64 and 32 respectively. After extracting the features by LSTM, we calculate the divergence between marginal distributions by a two-Gaussian-kernel MMD, and calculate the copula distance with respect to KL divergence. We run the model for 100 trials, and each trial costs about 10-20 minutes.

\bibliographystyle{IEEEtran}
\bibliography{CDAN.bib}


 




\vfill

\end{document}